\documentclass[accepted]{uai2026} % for initial submission
\usepackage[american]{babel}
\usepackage{natbib} % has a nice set of citation styles and commands
    \renewcommand{\bibsection}{\subsubsection*{References}}
\usepackage{mathtools} % amsmath with fixes and additions
\usepackage{booktabs} % commands to create good-looking tables
\usepackage{amsmath}
\usepackage{amssymb}
\usepackage{mathtools}
\usepackage{amsthm}

\usepackage{epstopdf,fancyhdr,amsmath,amsthm,xspace,enumerate,paralist, mathtools,tabularx,hhline,tabu,forest,array,stmaryrd}
\usepackage{pgf,tikz}
\usepackage{bm}
\usepackage{subcaption}
\usepackage{mathrsfs}
\usepackage{type1cm}
\usepackage{csquotes}
\usepackage{amssymb}
\usepackage{multirow}
\usepackage{verbatim}
\usepackage{enumitem}
\usepackage{upgreek}
\usepackage{natbib}
\usepackage{upgreek}
\usepackage{pdflscape}
\usepackage{thm-restate}
\usepackage{soul}
\usepackage{dsfont}
\usepackage{stmaryrd}
\usepackage{environ}
\usepackage{wrapfig}
\usepackage{float}
\usepackage{thm-restate}

\usepackage[capitalize,noabbrev]{cleveref}

\theoremstyle{plain}
\newtheorem{theorem}{Theorem}[section]

\theoremstyle{definition}
\newtheorem{definition}[theorem]{Definition}

\theoremstyle{remark}

\theoremstyle{definition}
\newtheorem{example}{Example}

\usepackage{pgfplots}
\DeclareUnicodeCharacter{2212}{−}
\usepgfplotslibrary{groupplots,dateplot}
\usetikzlibrary{patterns,shapes.arrows}
\pgfplotsset{compat=newest}

\usetikzlibrary{positioning,arrows,shapes,shapes.arrows,arrows.meta,trees,shapes.misc,shapes.geometric,decorations.pathreplacing,automata}
\tikzset{%
  >={Latex[width=1.5mm,length=2mm]},
  vertex/.style={draw,circle,inner sep=0mm,semithick,minimum width=4mm},
  point/.style = {circle, draw, inner sep=0.04cm,fill,node contents={}},
  uvertex/.style={outer sep=0},
  bidir/.style={<->,dashed, line width=0.25mm},
  dir/.style={->, line width=0.25mm},
  regime/.style={shape=rectangle,fill=black,inner sep=0pt,minimum size=3pt,draw},
  node distance=1cm,
  font=\scriptsize\sffamily
}

\definecolor{betterred}{RGB}{228,26,28}
\definecolor{betterblue}{RGB}{55,126,184}

\pgfdeclarelayer{back}
\pgfsetlayers{back,main}

\makeatletter
\pgfkeys{%
  /tikz/on layer/.code={
    \def\tikz@path@do@at@end{\endpgfonlayer\endgroup\tikz@path@do@at@end}%
    \pgfonlayer{#1}\begingroup%
  }%
}
\makeatother

\makeatletter
\newcommand{\xdashleftrightarrow}[2][]{\ext@arrow 3359\leftrightarrowfill@@{#1}{#2}}
\def\rightarrowfill@@{\arrowfill@@\relax\relbar\rightarrow}
\def\leftarrowfill@@{\arrowfill@@\leftarrow\relbar\relax}
\def\leftrightarrowfill@@{\arrowfill@@\leftarrow\relbar\rightarrow}
\def\arrowfill@@#1#2#3#4{%
  $\m@th\thickmuskip0mu\medmuskip\thickmuskip\thinmuskip\thickmuskip
   \relax#4#1
   \xleaders\hbox{$#4#2$}\hfill
   #3$%
}
\makeatother

\newcolumntype{C}[1]{>{\centering\arraybackslash}m{#1}}

\newcommand{\Braces}[1]{\left\{ #1\right\}}

\newcommand{\Brackets}[2][]{#1\left[#2 #1\right]}

\newcommand{\Parens}[1]{\left(#1\right)}

\makeatletter
\def\rightarrowfill@@{\arrowfill@@\relax\relbar\rightarrow}
\def\leftarrowfill@@{\arrowfill@@\leftarrow\relbar\relax}
\def\leftrightarrowfill@@{\arrowfill@@\leftarrow\relbar\rightarrow}
\def\arrowfill@@#1#2#3#4{%
  $\m@th\thickmuskip0mu\medmuskip\thickmuskip\thinmuskip\thickmuskip
   \relax#4#1
   \xleaders\hbox{$#4#2$}\hfill
   #3$%
}
\makeatother

\newcommand{\I}{\mathds{1}}

\newcommand{\inv}[2]{P_{#2}\left ( #1\right)}
\newcommand{\invE}[2]{\3E_{#2}\left [ #1\right]}

\newcommand{\doo}{\text{do}}

\def\*#1{\boldsymbol{#1}}
\def\1#1{\mathcal{#1}}
\def\2#1{\mathscr{#1}}
\def\3#1{\mathbb{#1}}
\def\4#1{\mathds{#1}}
\def\5#1{\bar{\*#1}}

\title{Causal Gaussian Processes for Robust Treatment Effect Evaluation with Unobserved Confounding}

\author[1]{Junzhe Zhang}
\author[1]{Jingyuan Chen}
\author[2]{Elias Bareinboim}

\affil[1]{%
    Department of Electrical Engineering and Computer Science\\
    Syracuse University
}

\affil[2]{%
    Department of Computer Science\\
    Columbia University
}

\begin{document}
\maketitle

\begin{abstract}
The presence of confounding bias poses a key challenge in policy evaluation, as the target causal effects of actions are not identifiable (i.e., underdetermined) from observational data. On the other hand, existing confounding-robust evaluation strategies require detailed prior knowledge about the environment or apply only to discrete treatments and outcomes. This paper investigates causal effect evaluation over the continuous domain from confounded observations, while requiring only basic temporal ordering between the treatment and the outcome. We introduce a universal discretization of the exogenous domains that approximates the observational and interventional distributions of any causal model with arbitrary accuracy using a finite number of latent states. Building on this newfound universal approximation property, we develop a novel family of Causal Gaussian process (CGP) models that effectively approximate the observational and interventional distributions of any causal model with confounded observations. %Finally, we validate the proposed approach through comprehensive experiments, demonstrating that it is confounding-robust and yields useful representations for downstream tasks.
\end{abstract}

\section{Introduction}
Gaussian process (GP) modeling is one of the popular approaches for regression problems, learning unknown functional mappings from input to output. It is a Bayesian non-parametric model \citep{gershman2012tutorial} that explicitly expresses the generative process underlying the observed data. By effectively using an infinite number of model parameters and adapting the model's complexity to the data, Gaussian process modeling can capture a wide range of complex function patterns between the input and output domains \citep{williams1998prediction}. For these reasons, it is particularly useful in situations where the learner has only weak parametric knowledge of the underlying function and the function is complex to evaluate analytically. Gaussian process models have been successfully applied across different fields of science, including cosmology \citep{shafieloo2012gaussian}, biology \citep{mcdowell2018clustering}, psychology \citep{bahg2020gaussian}, and chemistry \citep{deringer2021gaussian}, to name a few.

Despite these advancements and successful applications of Gaussian process modeling, significant challenges remain when applying this approach to imperfect, biased observational data. While Gaussian process regression is effective at capturing and mimicking complex patterns in observations, it can often learn and amplify existing biases, leading to deviations from the underlying function. This vulnerability occurs across various types of data biases, including unmeasured confounders, selection bias, missing data, and generalization to heterogeneous environments. The following example illustrates these challenges in the context of unobserved confounding bias.

\begin{figure*}[t]
\centering
    \hfill
        \begin{subfigure}{0.25\linewidth}\centering
            \setlength{\abovecaptionskip}{0pt}
		\includegraphics[width=\linewidth]{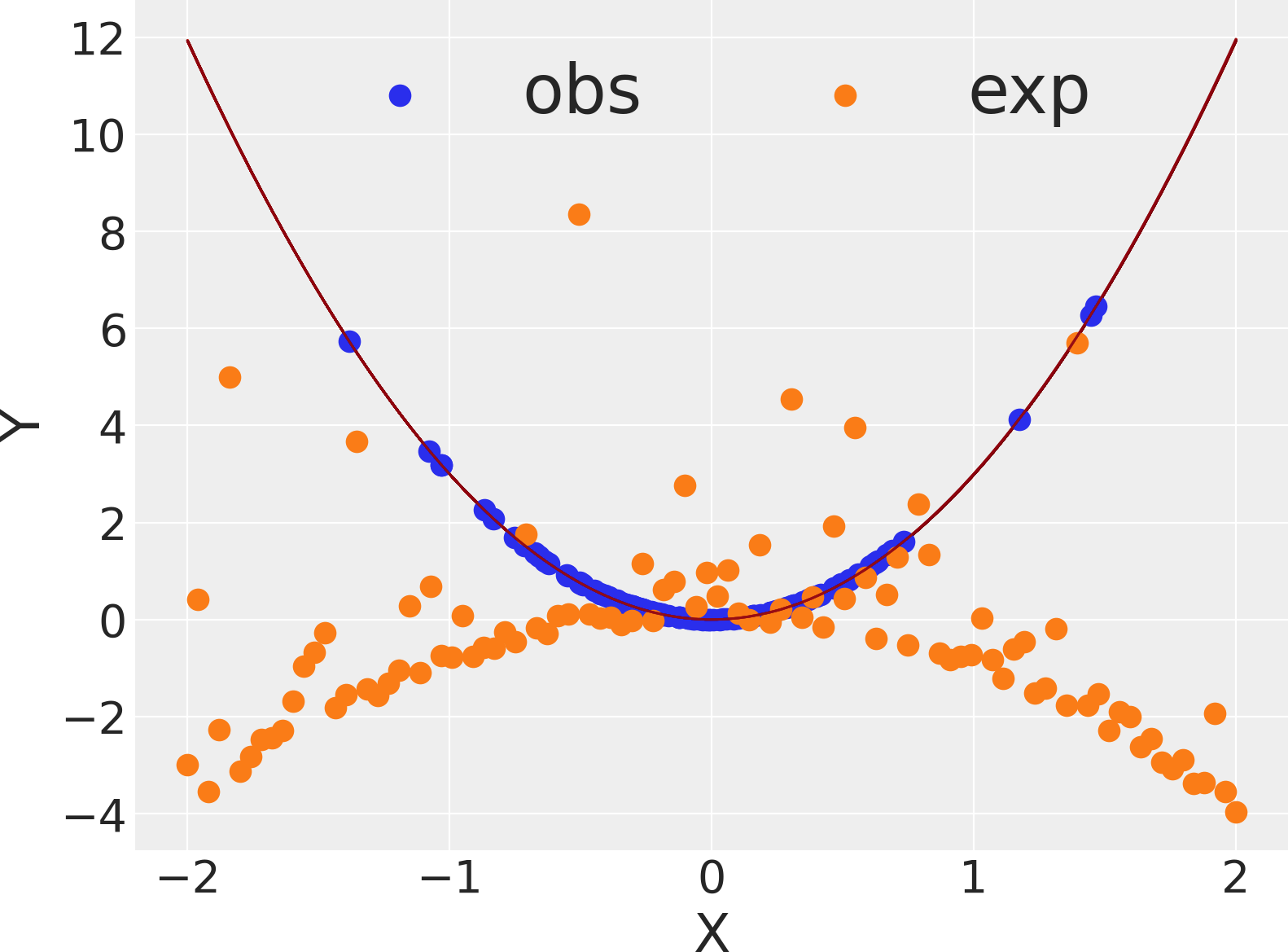}
		\caption{Polynomial}
		\label{fig:gp_confounded}
	\end{subfigure}\hfill
        \begin{subfigure}{0.25\linewidth}\centering
            \setlength{\abovecaptionskip}{0pt}
		\includegraphics[width=\linewidth]{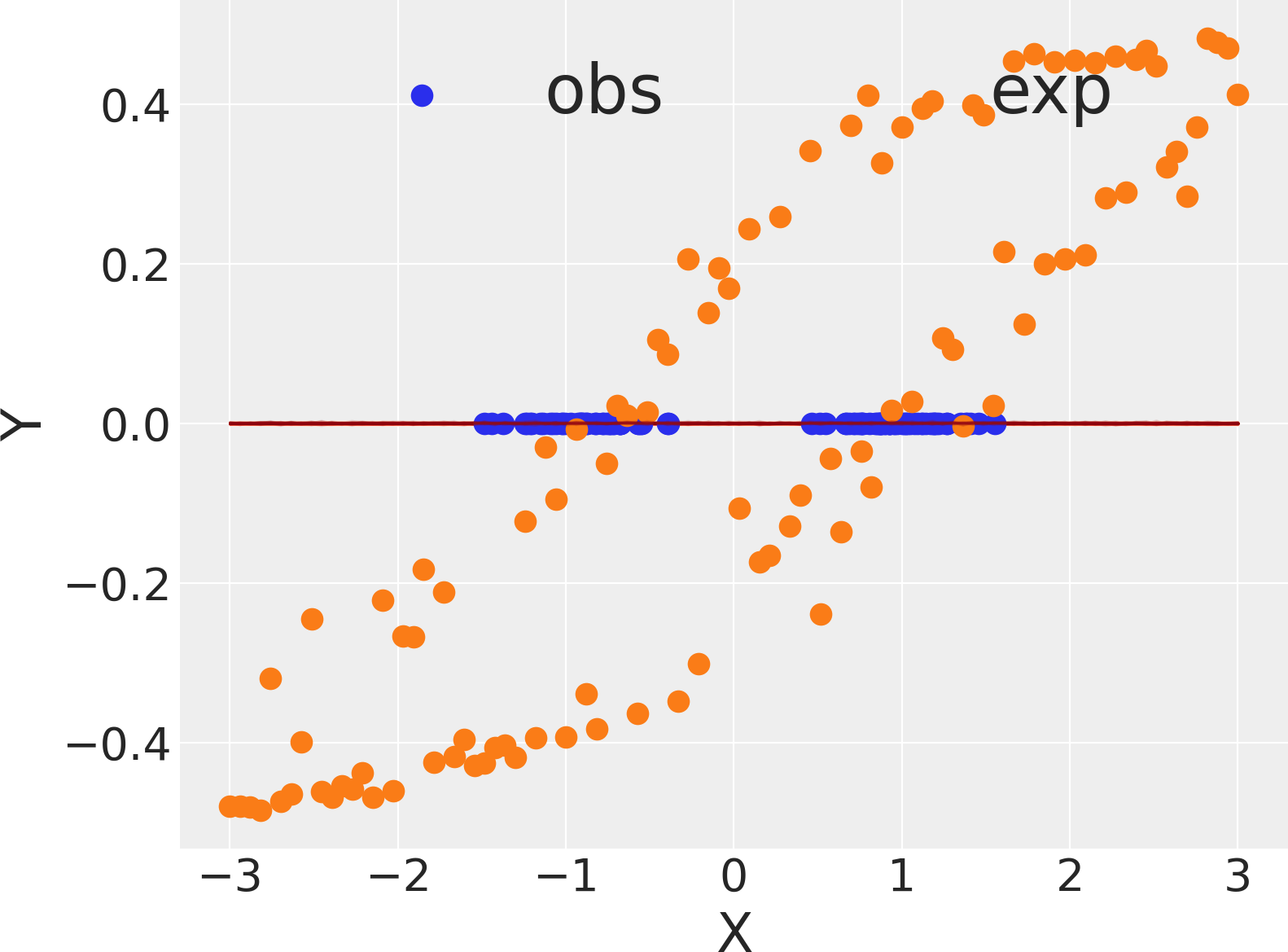}
		\caption{Logistic}
		\label{fig:gp_confounded_logistic}
	\end{subfigure}\hfill
        \begin{subfigure}{0.25\linewidth}\centering
            \setlength{\abovecaptionskip}{0pt}
		\includegraphics[width=\linewidth]{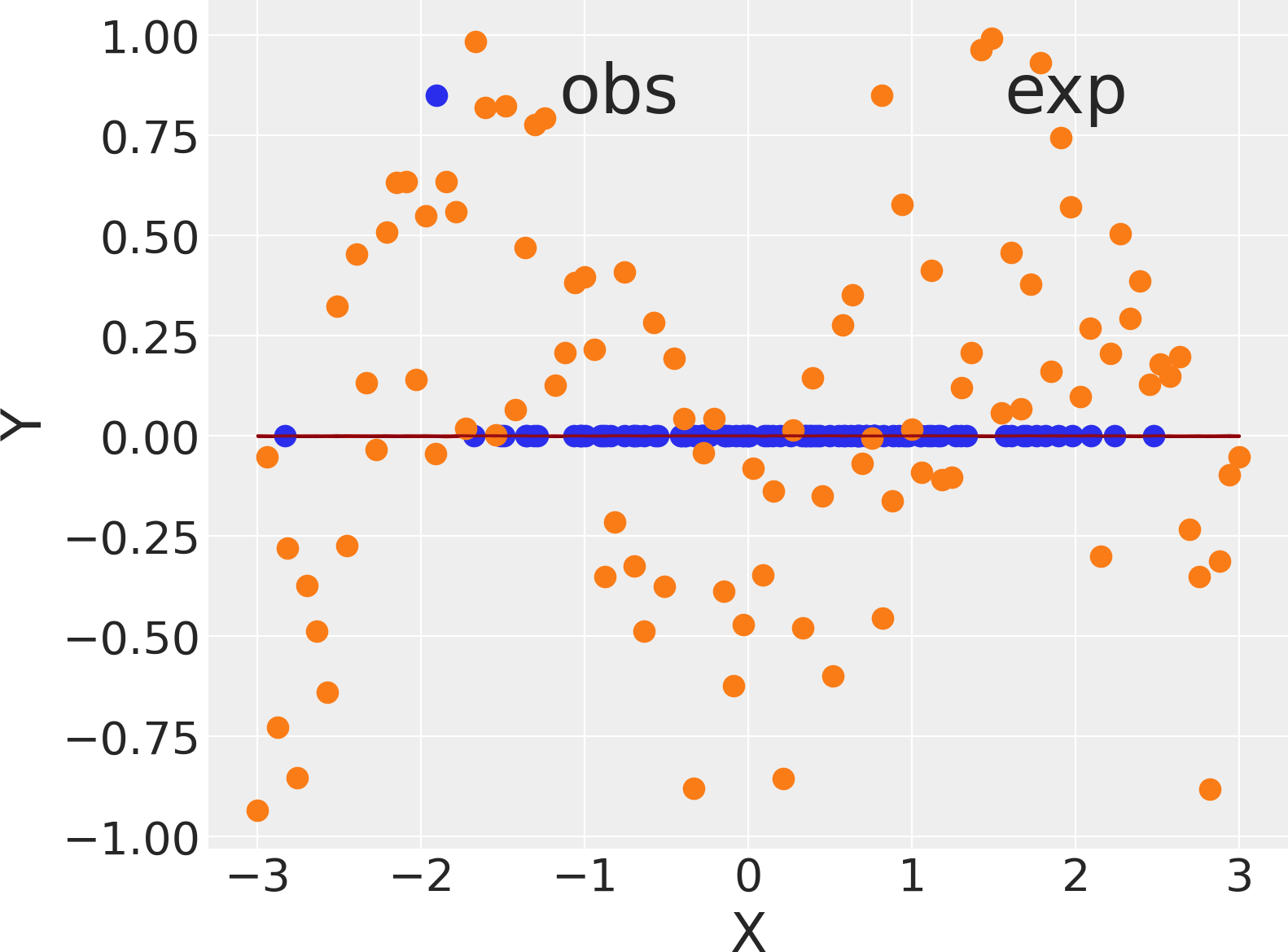}
		\caption{Phase}
		\label{fig:gp_confounded_phase}
	\end{subfigure}\hfill
    \begin{subfigure}{0.25\linewidth}\centering
            \setlength{\abovecaptionskip}{0pt}
		\includegraphics[width=\linewidth]{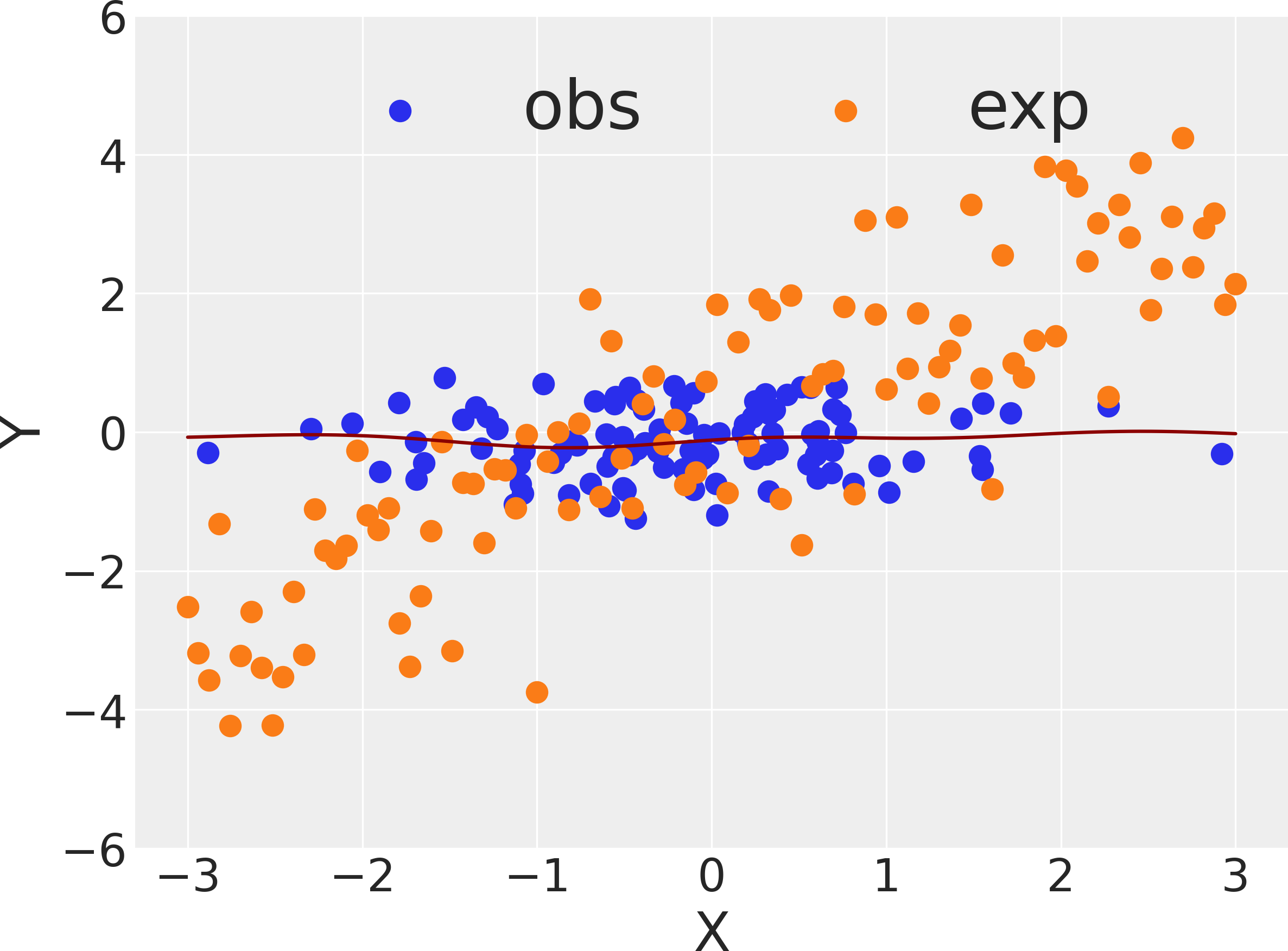}
		\caption{Linear}
        \label{fig:gp_confounded_linear}
		% \label{fig:gp_confounded_ist}
	\end{subfigure}\hfill\null
  \caption{Samples drawn from the observational $P(X, Y)$ (\textcolor{blue}{blue}) and interventional $P_x(Y)$ (\textcolor{orange}{orange}) distributions defined by various reward functions. These functions include: (\subref{fig:gp_confounded}) polynomial function; (\subref{fig:gp_confounded_logistic}) logistic function; (\subref{fig:gp_confounded_phase}) phase function; and (\subref{fig:gp_confounded_linear}) linear function. The regression function is obtained by applying a Gaussian Process model to the observed data.}
  \label{fig:gp}
\end{figure*}

\begin{example}\label{exp:_1}
Consider a data-generating process concerning a system with an action $X$ and a reward $Y$, values of which are decided by functions $X \gets - U / 2$ and $Y \gets -X^2 + U^2$, respectively; $U$ is an unobserved variable drawn from normal distributions with mean $\mu = 0$ and variances $\sigma^2 = 1$. \Cref{fig:gp_confounded} shows the observed samples (highlighted in \textcolor{blue}{blue}) generated by this system, summarized as the observational distribution $P(X, Y)$. We also show samples (\textcolor{orange}{orange}) collected by randomly assigning action values over a real interval $[-2, 2]$, summarized as an interventional distribution $\inv{Y}{x}$. Interestingly, the observational distribution $P(X, Y)$ deviates significantly from the interventional distribution $\inv{Y}{x}$. This is due to the unobserved confounder $U$, which introduces a spurious correlation between the treatment $X$ and the outcome $Y$, leading some actions to appear more effective than they actually are. 

It has been acknowledged in the literature that treatment effects $\invE{Y}{x}$ are not identifiable in such systems \citep{huang:val06,shpitser:pea06a}. Particularly, one could construct an alternative system with reward function $Y \gets 3/4 \times U^2$ that generates the same observed data, but gives a different evaluation on the treatment effects $\invE{Y}{x}$. To witness, we apply standard Gaussian process regression to the observed samples. The learned function, shown in \Cref{fig:gp_confounded}, perfectly fits the conditional reward $\3E[Y\mid x]$. However, it fails to generalize to the actual treatment effect $\invE{Y}{x}$ when one actively intervenes in the system by setting the action $X$ to a constant $x \in [-2, 2]$. $\hfill \blacksquare$
\end{example}
This paper focuses on evaluating the causal effects of treatments from observed data that may be contaminated with unobserved confounding bias.  The problem of identifying causal effects from the combination of observed data and theoretical assumptions about the underlying generative process has been extensively studied under the rubrics of causal inference  \citep{wright:28,angrist:etal96,pearl:2k,spirtes:etal01}. Particularly, qualitative causal knowledge about the underlying environment could be represented in the form of a \textit{directed acyclic causal diagram} \citep[Ch.~1.2]{pearl:2k}. Various criteria and algorithms have been developed based on the causal diagram to identify causal effects. \citep{pearl:2k,spirtes2000causation,bareinboim:pea:16-r450}. This means the conditions under which the target effects are fully recoverable from data have been understood. For example, a criterion called \textit{back-door} \citep[Ch.~3.2.2]{pearl:2k} allows one to identify causal effects by covariate adjustment. This condition is also referred to as \textit{ignorability} \citep{rosenbaum:rub83}. Efficient estimators were developed based on the inverse propensity score weighting \citep{rosenbaum:rub83,bang2005doubly} and off-policy learning \citep{dudik2011doubly,li2015toward,munos2016safe,thomas2016data}.

However, in many practical applications, the combination of qualitative knowledge and observed data does not always allow one to uniquely determine target effects, called \emph{non-identifiable}. Partial identification is a line of methods that allows one to extrapolate from biased, confounded data to infer partial knowledge about non-identifiable causal effects \citep{balke:pea97,frangakis:rub02}. For instance, one may incorporate additional parametric assumptions about the forms of underlying functions and distributions \citep{wright1934method,manski1998monotone,kocaoglu2018causalgan}. However, these assumptions are often untestable and could incur high costs in knowledge engineering. More recently, there has been an increasing body of work studying the approximation property of generative models \citep{zhang2022partial,xia2021causal,xia2022neural,nasr2023counterfactual} in evaluating causal effects, particularly in non-identifiable settings. For instance, for an unknown causal model with discrete observed variables, \citep{zhang2022partial} showed that the domain of latent confounders could be discretized without loss of generality. Identifying unknown causal effects in this class of discrete generative models reduces to solving a series of polynomial programs, which can be further simplified to linear programs in specific settings \citep{balke:pea97}. However, significant challenges remain in applying causal generative modeling to observed data in continuous, complex domains. 

%For example, efficient algorithms exist to identify causal effects from confounded observations provided by the linearity assumption \citep{wright1934method,chen2017identification}. For non-linear functions, parametric conditions have been proposed under which the target effect is identifiable \citep{wang2019blessings,manski1998monotone}. In general, in situations where one could fully express the parametric forms of the generative process for the observed data and latent confounders, the computational framework of deep generative models \citep{goodfellowGenerativeAdversarialNetworks2014a,vahdat2020nvae} and Bayesian non-parametric methods \citep{gershman2012tutorial,blei2017variational} could be applied to obtain an equivalence class of candidate models compatible with observations. Computing the target causal effects in these candidate models thus yields a robust evaluation containing the unknown parameters \citep{kocaoglu2018causalgan,witty2020causal,khemakhem2021causal}. 

%Nevertheless, the causal inference methods described above invoke a detailed parametric model of the underlying causal mechanisms generating the data, which are often untestable in many problem settings, and could incur high costs in knowledge engineering. 

The goal of this paper is to address these challenges by proposing novel Causal Gaussian Processes that can approximate observational and interventional distributions in any unknown causal model with continuous treatment and outcome variables. This newfound approximation property enables us to develop robust partial identification algorithms for inferring unknown causal effects from confounded observational data in continuous domains. In some ways, this result can be seen as generalizing the canonical partitioning of structural causal models \citep{balke:pea94b} from discrete to continuous domains. 
More specifically, our contributions are summarized as follows. (1) We introduce a universal discretization procedure for the exogenous domain for any causal model with continuous treatment and outcome domains. We formally show that the proposed discretization can effectively approximate the observational and interventional distributions in any causal model in continuous domains with arbitrary accuracy. (2) We reparameterize this canonical representation into a novel family of causal GP models. Learning in this generative model family from confounded observations yields a robust posterior distribution over the parameters of the target causal effects. Due to the space constraints, all proofs of the theoretical results are provided in \Cref{appendix:a}; details about the experimental setup are provided in \Cref{appendix:b}.

\section{Evaluating Treatment Effect From Confounded Observations}\label{sec:_2}
This section introduces basic notations and definitions used throughout the paper. We use capital letters to denote variables ($X$) and small letters for their values ($x$). For an arbitrary set $\*X$, let $|\*X|$ be its cardinality. $P(\*X)$ denote the probability distribution on variables $\*X$. We consistently use $P(\*x)$ as a shorthand notation for probability $P(\*X = \*x)$; similarly, $P(\3X)$ stands for the probability $P(\*X \in \3X)$ of an event where values of variables $\*X$ are contained in a collection $\3X$ of possible realizations $\*X = \*x$.

\textbf{Structural Causal Models.} We will focus on the structural causal models \citep{pearl:2k,bareinboim2020pearl} graphically described in \Cref{fig:_mab_obs} where $\*X$ is is a treatment/action; $Y$ is the outcome/reward; $\*Z$ is a set of covariates; and $\*U$ are unobserved exogenous variables representing uncertainties in the environment. For an arbitrary causal model $\1M$, the values of treatment $\*X$, outcome $Y$ and covariates $\*Z$ are decided by structural functions $\*X \gets f_X(\*Z, \*U)$, $Y \gets f_Y(\*X, \*U)$ and $\*Z \gets f_Z(\*U)$ respectively. The causal mechanisms generating latent variables $\*U$ are not explicitly described. Instead, values of $\*U$ are drawn from an exogenous distribution $P(\*U)$ over the $m$-dimensional real space $\3R^m$, for $m \in \3N^+$. The learner does not have access to detailed knowledge of the underlying causal model $\1M$. Instead, it only passively observes the environment and receives action-reward pairs $(\*x, y)$. The probabilities of occurrence of observed events are summarized as the \emph{observational distribution} $P(\*X, Y, \*Z) \triangleq P(\*X, Y, \*Z; \1M)$.

\begin{figure}[t]
\centering
\hfill
\begin{subfigure}{0.48\linewidth}\centering%(a)
  \begin{tikzpicture}
      \def\innerr{2.7}

      \node[vertex] (X) at (0, 0) {X};
      \node[vertex] (Z) at (1, 1.) {Z};
      \node[vertex] (Y) at (2, 0) {Y};

      \draw[dir] (Z) to (Y);
      \draw[dir] (X) -- (Y);
      \draw[dir] (Z) to [bend right = 0] (X);
      \draw[bidir, dashed] (Z) to [bend right = 45] (X);
      \draw[bidir, dashed] (Z) to [bend left = 45] (Y);
      \draw[bidir, dashed] (X) to [bend right = 45] (Y);

  \begin{pgfonlayer}{back}
      \node[circle,fill=betterred!65,draw=none,minimum size=2*\innerr mm] at (X) {};
      \node[circle,fill=betterblue!65,draw=none,minimum size=2*\innerr mm] at (Y) {};
  \end{pgfonlayer}
  \end{tikzpicture}
  \caption{$\1M$}
  \label{fig:_mab_obs}
  \end{subfigure}\hfill
\begin{subfigure}{0.48\linewidth}\centering
\begin{tikzpicture}
      \def\innerr{2.7}

            \node[vertex] (X) at (0, 0) {X};
      \node[vertex] (Z) at (1, 1.) {Z};
      \node[vertex] (Y) at (2, 0) {Y};

      \draw[dir] (Z) to (Y);
      \draw[dir] (X) -- (Y);
      \draw[bidir, dashed] (Z) to [bend left = 45] (Y);
      \draw[bidir, dashed, opacity=0] (X) to [bend right = 45] (Y);

  \begin{pgfonlayer}{back}
      \node[circle,fill=betterred!65,draw=none,minimum size=2*\innerr mm] at (X) {};
      \node[circle,fill=betterblue!65,draw=none,minimum size=2*\innerr mm] at (Y) {};
  \end{pgfonlayer}
  \end{tikzpicture}
    \caption{$\1M_{\*x}$}
   \label{fig:_mab_exp}
  \end{subfigure}\hfill \null
  \caption{Causal diagrams of (a) a contextual bandit model with a treatment/action $\*X$, outcome/reward $Y$, and covariates $\*Z$; (b) the submodel induced by intervention $\doo(\*X \gets \*x)$. Exogenous variables $\*U$ are not explicitly shown; instead, a bi-directed arrow $\*X \leftrightarrow Y$ indicates an unobserved confounder $\*U$ affecting $\*X$ and $Y$ simultaneously.}
  \label{fig:bow}
\end{figure}
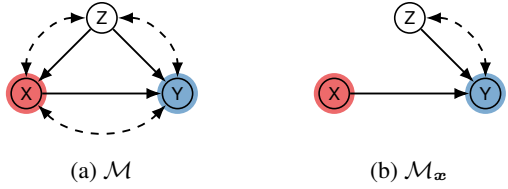
An intervention on action $\*X$, denoted by $\doo(\*X \gets \*x)$ (for short, $\doo(\*x)$), is an operation which sets values of $\*X$ to constants $\*x$, replacing the function $f_X$ that would normally determine the action. For a structural causal model $\1M$, let $\1M_{\*x}$ be a submodel of $\1M$ induced by intervention $\doo(\*x)$. Given unit $\*U = \*u$, the \emph{potential outcome} $Y$ to action $\doo(\*x)$, denoted as $Y_{\*x}(\*u)$, is the solution of a variable $Y$ in submodel $\1M_{\*x}$, i.e., $Y_{\*x}(\*u) \triangleq Y_{\1M_{\*x}}(\*u)$. The \emph{interventional distribution} $\inv{Y, \*Z}{\*x}$ induced by $\doo(\*x)$ is defined as the joint distribution over the potential outcome, i.e., $\inv{Y, \*Z}{\*x} \triangleq P\Parens{Y_{\*x}, \*Z_{\*x}; \1M}$. The conditional treatment effect of action $\*X = \*x$ in situation $\*Z = \*z$ is thus defined as the conditional expected value $\invE{Y \mid \*z}{\*x}$.

\textbf{Treatment Evaluation Problem.} In this paper, our goal is to infer the conditional treatment effect $\invE{Y \mid \*z}{\*x}$ for all possible actions $\*x$ in all possible situations $\*z$ (i.e., \textbf{output}) from the finite samples drawn from the observational distribution $P(\*X, Y, \*Z)$ (i.e., \textbf{input}). We will also make the following \textbf{assumptions}. First, the treatment $\*X$ and covariate $\*Z$ are vector variables in a multi-dimensional real space $\*X \in \3R^n$ and $\*Z \in \3R^l$ respectively; the reward $Y$ is a continuous variable taking a real value in $Y \in \3R$. The target treatment effect $\invE{Y \mid \*z}{\*x}$ is Lebesgue-integrable. That is, $\invE{Y \mid \*z }{\*x} : \3R^{n} \times \3R^l \mapsto \3R$ is a Lebesgue-measurable function satisfying the following relationship,
\begin{align}
\int_{\3R^n \times \3R^l} \left \lvert \invE{Y \mid \*z}{\*x} \right \rvert d\*z d\*x < \infty,
\end{align} 
which includes continuous functions, including functions such as the $\text{sgn}$ function. Because we deal with Lebesgue-integrable functions, we adopt the $L_1$ distance as a measure of approximation error.
%We will consistently use the $L_1$ norm $\lVert \*x \rVert_1 = \sum_{i = 1}^{n} |x_i|$ to measure approximation error.

When the parametric knowledge is available for the underlying structural functions $f_X, f_Y, f_Z$ and exogenous distribution $P(\*u)$, one could construct the feasible region of all possible causal models compatible with the observational data $P(\*X, Y, \*Z)$ and causal knowledge (i.e., the temporal ordering of $\*X, Y, \*Z)$. Searching over models in the feasible region and computing their induced treatment effects leads to a robust evaluation of the target causal query. Several Bayesian algorithms and frameworks have been proposed \citep{geman1984stochastic,kingma2013auto,blei2017variational}. However, in practice, detailed parametric knowledge of the underlying generative process is not available, posing challenges for applying existing Bayesian methods to robust causal inference from confounded data.
\section{Canonical Partitioning of Structural Causal Models}
Our goal in this section is to address the aforementioned challenges by introducing a parametric family of canonical generative processes that could approximate the observational distribution and the target treatment effect in any structural causal model described in \Cref{fig:_mab_obs}. Specifically, this canonical family allows one to construct an equivalent class of parametric models that are compatible with the observational data and available causal knowledge. \Cref{fig:bounding} illustrates this learning process. Since the constructed canonical models are expressive enough to represent the observations and effects of interventions in any candidate causal models, there is no loss of generality when inferring causal effects in this canonical family. More importantly, since the parametric forms of the generative processes in these canonical models are explicitly specified, the learner can apply the standard computational frameworks in Bayesian inference and deep learning to infer the posterior causal effect from available data and knowledge. 

\begin{figure}[t]
\centering
      \begin{tikzpicture}[-, scale=0.8, every node/.append style={transform shape}]
        \filldraw [fill=gray!50, line width=0.01mm] (-2.05,-0.05) ellipse (2.0 and 1.0);
        \filldraw [fill=gray!20, line width=0.01mm] (-1.6,0.1) ellipse (1.2 and 0.6);
      	\node [font=\fontsize{9}{0}\selectfont] at (-3.8,-1.1) {True Model};
            \node [inner sep=0] (scm) at (-3.8,-0.9) {};
      	\node [inner sep=0] (scmspace) at (-3.3,-0.5) {};
      	\path [-] (scm) edge (scmspace);
       
      	\node [align=center, font=\fontsize{9}{0}\selectfont] at (-1.4,1.3) {Causal Canonical Models};
      	\node [inner sep=0] (omg) at (-1.4,1.05) {};
      	\node [inner sep=0] (omgspace) at (-1.4,0.5) {};
      	\path [-] (omg) edge (omgspace);
      	
      	\node at (-2.85,-0.4) {$\1M^*$};
      	\draw [fill=black] (-2.75,-0.7) circle (0.1);
      	\node [inner sep=0] (ms) at (-2.75,-0.7) {\ };

      	\draw [fill=black] (-1.5,-0.2) circle (0.1);
      	\node [inner sep=0] (mp) at (-1.5,-0.2) {};
      	\node at (-1.6,0.2) {$\widehat{\1M}$};
       
      	\draw [densely dotted] (-2.1,0.35) circle (0.1);
      	\draw [densely dotted] (-1.85,-0.15) circle (0.1);
      	\draw [densely dotted] (-1.15,-0.15) circle (0.1);
      	\draw [densely dotted] (-0.8,0.2) circle (0.1);
      	
      	\draw [densely dotted] (-3.0,0.4) circle (0.1);
        \draw [densely dotted] (-3.5,0.25) circle (0.1);
        \draw [densely dotted] (-3.25,0.0) circle (0.1);
        \draw [densely dotted] (-3.8, 0.0) circle (0.1);
        \draw [densely dotted] (-3.4,-0.4) circle (0.1);
        \draw [densely dotted] (-2.2,-0.8) circle (0.1);
        \draw [densely dotted] (-1.7,-0.7) circle (0.1);
      	\draw [densely dotted] (-2.4,0) circle (0.1);
      	\draw [densely dotted] (-1.1,0.5) circle (0.1);

      	\draw (-2.05,-2.72) ellipse (1.77 and 0.8);
        \draw [densely dotted] (-2.05,-2.72) ellipse (1.4 and 0.55);
      	\node [align=center, font=\fontsize{9}{0}\selectfont] (l1) at (-2,-3.8) {Observational Data};
      	\node at (-2.05,-2.8) {\small $P^*\!(\*V){=}\widehat{P}(\*V)$};
      	\draw [fill=black] (-2.05,-2.5) circle (0.05);
      	\node [inner sep=0.2em] (pl1) at (-2.05,-2.5) {\ };
    
      	\draw (2.8, 0.1) ellipse (1.4 and 0.7);
      	\node [align=center, font=\fontsize{9}{0}\selectfont] (l2t) at (2.9,-0.8) {Treatment Effect};
      	\draw [densely dotted] (2.8,0.1) ellipse (1.1 and 0.5);
      	\node [align=center] at (2.8,0.1) {{\small $\widehat{\3E}_{\*x}\left[Y \mid \*z\right]$}};
      	\draw [fill=black] (2.8,-0.3) circle (0.05);
      	\node [inner sep=0.2em] (pl21) at (2.8,-0.3) {\ };

      	\path [-Latex] (ms) edge (pl1);
      	\path [-Latex] (mp) edge [bend left = 15](pl1); 
            \path [-Latex] (mp) edge  node[above, font=\fontsize{9}{0}\selectfont]  {infer} (pl21);
            
      	\draw (2.0,-2.82) ellipse (1 and 0.6);
      	\node [align=center, font=\fontsize{9}{0}\selectfont] at (2.0,-3.8) {Causal Knowledge};
      	\draw [densely dotted] (2,-2.82) ellipse (0.8 and 0.45);
      	\node [align=center] at (2,-2.9) {$\1G = \widehat{\1G}$};
      	\draw [fill=black] (2,-2.7) circle (0.05);
      	\node [inner sep=0] (g1) at (2,-2.7) {};
      	\path [-Latex] (mp) edge [bend left = 15](g1);
      	\path [-Latex] (ms) edge (g1);

            \filldraw [fill=none,line width=0.01mm] (1.2, -1.1) rectangle (4.5, 1.1);
            \filldraw [fill=none,line width=0.01mm] (-4, -4) rectangle (3.5, -1.8);
            
            \node [align=center, font=\fontsize{8}{0}\selectfont, fill=white] (gconstrainttext) at (3.5, 1.1) {Posterior};
            \node [align=center, font=\fontsize{8}{0}\selectfont, fill=white] (gconstrainttext) at (2.5, -1.8) {Prior};

    \end{tikzpicture}
\caption{The illustration of the confounding-robust inference of the unknown treatment effect from the combination of observational data and causal knowledge.}\label{fig:bounding}
\end{figure}
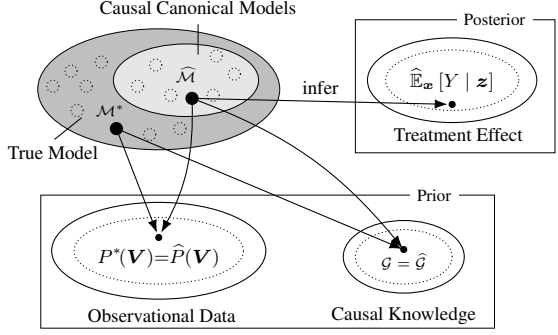

The key to constructing canonical models is a partitioning of the latent exogenous domains. To formally introduce this procedure, we first describe some necessary notation. We will consistently use $\3U$ to denote a subset of the exogenous domain $\3R^m$, $\3X$ for a subset of the action domain $\3R^n$, and $\3Z$ for a subset of the covariate domain $\3R^l$. Fix an arbitrary integer $N \in \3N^+$. Let a sequence of constants $\*z_1, \dots, \*z_N$ in the covariate domain $\3R^l$, and let $\3U_1, \dots, \3U_N$ be disjoint subsets in the exogenous domain $\3R^m$. Similarly, let $\*x_1, \dots, \*x_N$ be a finite sequence of constants in the action domain $\3R^n$,  $y_1, \dots, y_N$ be constants in the reward domain $\3R$, and $\3X_1, \dots, \3X_N$ be disjoint subsets in the action domain $\3R^n$. Simple functions determining values of covariate $\*Z$, treatment $\*X$ and outcome $Y$ are given by, respecrtively,
\begin{align}
    &\widehat{f}_{\*Z}(\*u) = \sum_{i = 1}^{N} \*z_i \I_{\3U_{i}}(\*u), \label{eq:simple_z}\\
    &\widehat{f}_{\*X}(\*z, \*u) = \sum_{i = 1}^{N} \*x_i \I_{\3Z_{i}}(\*z) \I_{\3U_{i}}(\*u),  \label{eq:simple_x} \\
    &\widehat{f}_Y(\*x, \*z, \*u) = \sum_{i = 1}^{N} y_i \I_{\3X_i}(\*x)\I_{\3Z_i}(\*z) \I_{\3U_{i}}(\*u), \label{eq:simple_y}
\end{align}
where $\I_{\3U}$, $\I_{\3X}$ and $\I_{\3z}$ are the indicator functions of the subset $\3U_i \subseteq \3R^m$, $\3X_i \subseteq \3R^n$ and $\3Z_i \subseteq \3R^l$ respectively. A causal canonical model compatible with the qualitative assumptions of \Cref{fig:_mab_obs} is an SCM where simple functions determine the values of its observed variables. Formally,
\begin{definition}\label{def:ccm}
    A causal canonical model (CCM) $\1M$ is an SCM where values of its covariate $\*Z$, action $\*X$, and reward $Y$ are decided by simple functions $\widehat{f}_{\*Z}$, $\widehat{f}_{\*X}$ and $\widehat{f}_Y$ defined in \Cref{eq:simple_z,eq:simple_x,eq:simple_y} respectively.
\end{definition}
Let $\2M$ be the set of all SCMs compatible with the causal graph in \cref{fig:_mab_obs}. Similarly, let $\2N$ denote the space of CCMs compatible with \cref{fig:_mab_obs}. By the definition of CCMs in \cref{def:ccm}, the canonical space $\2N$ must be strictly contained in the original space $\2M$. We want to identify a subspace of causal models contained in $\2M$ with the following \textit{causal approximation property}, i.e.,
\begin{definition}[Causal Approximation Property]\label{def:approx}
    Let $\2M$ be the set of all SCMs described in \Cref{fig:_mab_obs}. A subset $\2N \subset \2M$ is said to satisfy the \emph{causal approximation property} if given any SCM $\1M \in \2M$ and any $\epsilon > 0$, there exists an alternative causal model $\widehat{\1M} \in \2N$ such that, for all bounded continuous functions $h: \3R^{n+l+1} \mapsto \3R$, the following conditions hold:
    \begin{align}
        &\left \lvert \3E \Brackets{h(\*X, \*Z, Y); \1M} - \3E \Brackets{h(\*X, \*Z, Y); \widehat{\1M}}  \right \rvert < \epsilon, \label{eq:approx_obs}\\
        &\int_{\3R^{n+l}} \left \lvert \invE{Y \mid \*z ; \1M}{\*x} - \invE{Y \mid \*z ; \widehat{\1M}}{\*x}  \right \rvert d\*z d\*x < \epsilon \label{eq:approx_do}
    \end{align}
    That is, the subset $\widehat{\1M}$ is \emph{dense} in the set of all SCMs $\1M$ with regard to (w.r.t.) the $L_1$ norm.
\end{definition}
One may surmise that since the parametric forms of simple functions are restrictive (compared to an arbitrary structural function), the family of causal canonical models is not sufficient in representing all the observational distributions (i.e., data) and treatment effects (query) in an arbitrary causal model. Our following result shows that this is not the case.
\begin{restatable}{theorem}{thmapprox}\label{thm:approx}
The set of CCMs $\widehat{\2M}$ is dense in the set of all SCMs $\2M$ compatible with \Cref{fig:_mab_obs}.
\end{restatable}
\Cref{thm:approx} says that for an arbitrary SCM $\1M$, there exists a causal canonical model $\widehat{\1M}$ that converges to $\1M$ in the observational distribution $P(\*X, \*Z, Y)$. For example, let $h$ be the product of indicator functions $h(\*x,\*z, y) = \I_{\*x \leq \*x'}\I_{\*z \leq \*z'}\I_{y \leq y'}$ for constants $\*x' \in \3R^n$, $\*z' \in \3R^l$ and $y \in \3R$. Among the equations above, the first term ensures that the cumulative observational distribution $P\Parens{\*X \leq \*x', \*Z \leq \*z', Y \leq y'}$ induced by the canonical model $\widehat{\1M}$ converges to the same distribution function in the causal model of ground truth $\1M$. In addition, the second term ensures that the canonical model $\widehat{\1M}$ could approximate the treatment effects $\invE{Y \mid \*z}{\*x}$ for every treatment $\*x \in \3R^n$ and context $\*z \in \3R^l$ with arbitrary precision.

\begin{example}\label{exp:_2} 
Consider the causal model $\1M$ described in \Cref{exp:_1}. We will construct a canonical model $\widehat{\1M}$ to approximate the observational distribution $P(X, Y)$ and treatment effects $\invE{Y}{x}$, $\forall x \in [-2, 2]$. The idea of the construction is to define a simple function $\widehat{f}_X$ determining values of treatment $X$ over the subintervals of the exogenous domains. For each integer $k \in \3Z$, we will decompose the exogenous domain $[-2, 2]$ into $2^{k+2}$ equally sized disjoint intervals. Specifically, we define $\3U_i(k)$ as a subinterval contained in $[-2, 2]$ given by,  for $i = 1, \dots, 2^{k+2}$,
\begin{align}
    \3U_i(k) = \Braces{u: -2 + \frac{i - 1}{2^{k}} \leq u < -2 + \frac{i}{2^{k}}} \label{eq:_4_3_1}
\end{align}
The function $\widehat{f}_X(u)$ is a linear combination of indicator functions over $\3U_i(k)$ given by, for $i = 1, \dots, 2^{k+2}$,
\begin{align}
    &\widehat{f}_X(u) = \begin{dcases}
    \frac{1}{2}\Parens{2 - \frac{i - 1}{2^k}} &\mbox{for }u \in \3U_i(k)
    \end{dcases}
\end{align}
Similarly, values of the outcome $Y$ are determined by a simple function $\widehat{f}_Y$ defined over the product of subintervals of the exogenous and treatment domains. Let $\3X_i(k)$ be a subinterval contained in $[-2, 2]$ defined as, for $i = 1, \dots, 2^{k+2}$,
\begin{align}
    \3X_i(k) = \Braces{x: -2 + \frac{i - 1}{2^{k}} \leq x < -2 + \frac{i}{2^{k}}}
\end{align}
The function $\widehat{f}_Y(u)$ is a linear function over the product of subintervals $\3U_i(k)$ and $\3X_i(k)$, with additional observed actions to ensure the canonical model is consistent in observational distribution. Formally, for $i,j = 1, \dots, 2^{k+2}$,
 \begin{align}
    &\widehat{f}_Y(x, u)\\
    &=\begin{dcases}
         \frac{3}{4}\Parens{-2 + \frac{j - 1}{2^k}}^2, \mbox{if } x = \widehat{f}_X(u) \text{ and } u \in \3U_j(k)\\
         -\Parens{-2 + \frac{i - 1}{2^k}}^2 + \Parens{-2 + \frac{j - 1}{2^k}}^2, \\
         \;\;\;\;\;\;\;\;\;\;\;\;\;\;\;\;\;\;\;\;\;\;\;\;\;\;\;\;\;\;   \mbox{if } x \in \3X_i(k) \text{ and } u \in \3U_j(k)
    \end{dcases} \notag
\end{align}
\end{example}
\begin{wrapfigure}[18]{r}{0.55\linewidth}
\centering
\vspace{-0.15in}
\includegraphics[width=\linewidth]{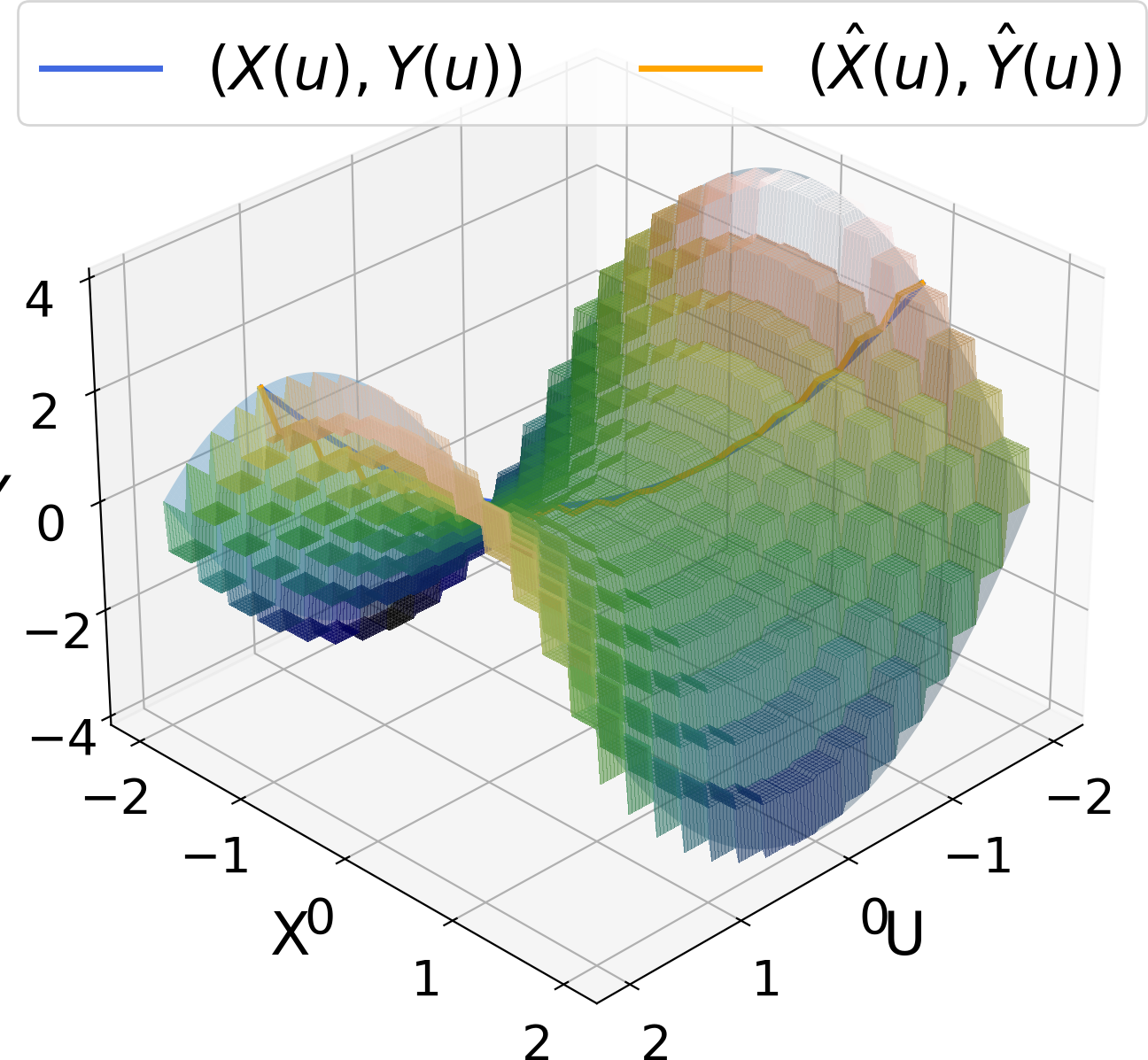}
  \caption{A simple function approximating the reward function $f_Y(x, u)$ in the ground-truth causal model $\1M$ of \Cref{exp:_1}. The observed trajectories and their approximations are highlighted in \textcolor{blue}{blue} and \textcolor{orange}{orange}.}
  \label{fig:_3_3_fy}
\end{wrapfigure}
Among the above equations, the first condition in $\widehat{f}_Y(x, u)$ ensures that the outcome variable $\widehat{Y}(u)$ in the canonical model $\widehat{\1M}$ effectively approximates the observation $Y(u)$ in the causal model $\1M$. It follows from the constructions of canonical observations $\widehat{X}(u)$ and $\widehat{Y}(u)$ that they converge to the observed variables $X(u) $ and $Y(u)$ in probability. We show in \Cref{fig:_3_3_fy} the graphical representation of the simple reward function $\widehat{f}_Y$ for $k = 5$, where the ground-truth reward function $f_Y$ is shown in a blue surface. We also highlight the observed trajectories $\left( X(u), Y(u)\right)$ in the causal model $\1M$ and the canonical model $\widehat{\1M}$ in \textcolor{blue}{blue} and \textcolor{orange}{orange} respectively. One could see by inspection that the observational distribution $P(\widehat{X}, \widehat{Y})$ in the canonical model $\widehat{\1M}$ must converge to the observational distribution $P(X, Y)$ in the causal model $\1M$ as the cardinality $k$ increases.

In terms of approximation of treatment effects, the definitions of $\widehat{f}_Y(x, u)$ ensure that for actions $x \neq \widehat{f}_X(u)$, the constructed $\widehat{Y}_x(u)$ converges to the original potential outcome $Y_x(u)$ in probability. This means that for actions $x \neq 1 - (i-1)2^{-k-1}, i = 1, \dots, 2^{k+2}$, the treatment effect $\widehat{\3E}_x[Y]$ in the canonical model $\widehat{\1M}$ converges to $\invE{Y}{x}$ in the original model $\1M$. Since the construction of $\widehat{\1M}$ only adds a finite number of points to the equivalence class of the action domain, these points have Lebesgue measure zero and do not affect the $L_1$ approximation error. $\hfill \blacksquare$
\section{Generative Modeling for Partial Causal Identification}
While the representation of causal canonical models is expressive enough to capture all possible observations and interventions, it remains challenging to infer target causal effects from confounded observations. The goal of this section is to overcome this challenge by developing a novel generative modeling approach for evaluating non-identifiable causal effects from confounded observations. We first introduce a reparametrization of the causal canonical models using continuous functions. Based on the reparametrization, we then discuss a novel partial identification method for evaluating causal effects via variational inference. 

\begin{figure*}[t]
\hfill
        \begin{subfigure}{0.25\linewidth}\centering
            \setlength{\abovecaptionskip}{0pt}
		\includegraphics[width=\linewidth]{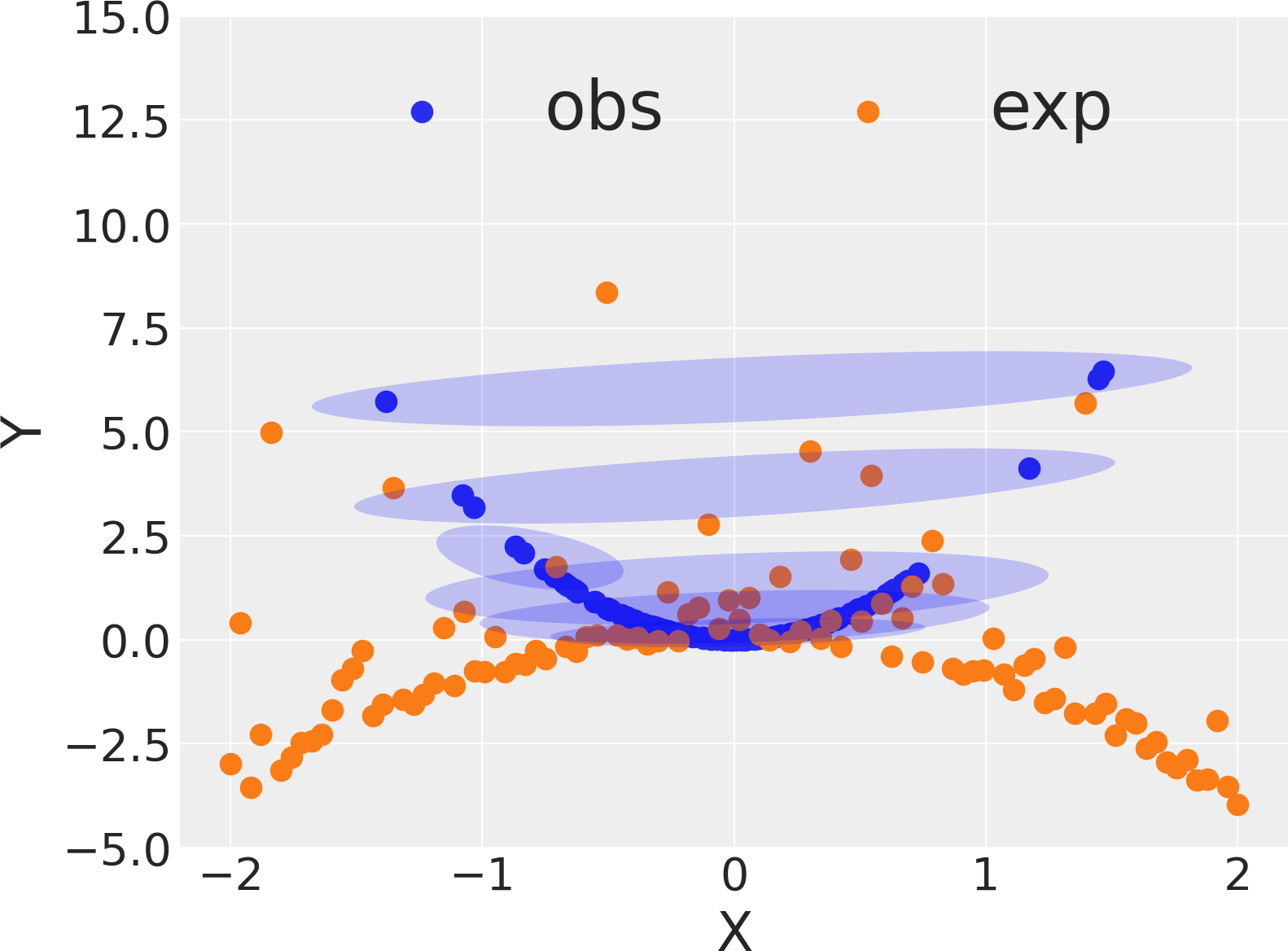}
		\caption{}
		\label{fig:_4_1_a}
	\end{subfigure}\hfill
        \begin{subfigure}{0.25\linewidth}\centering
            \setlength{\abovecaptionskip}{0pt}
		\includegraphics[width=\linewidth]{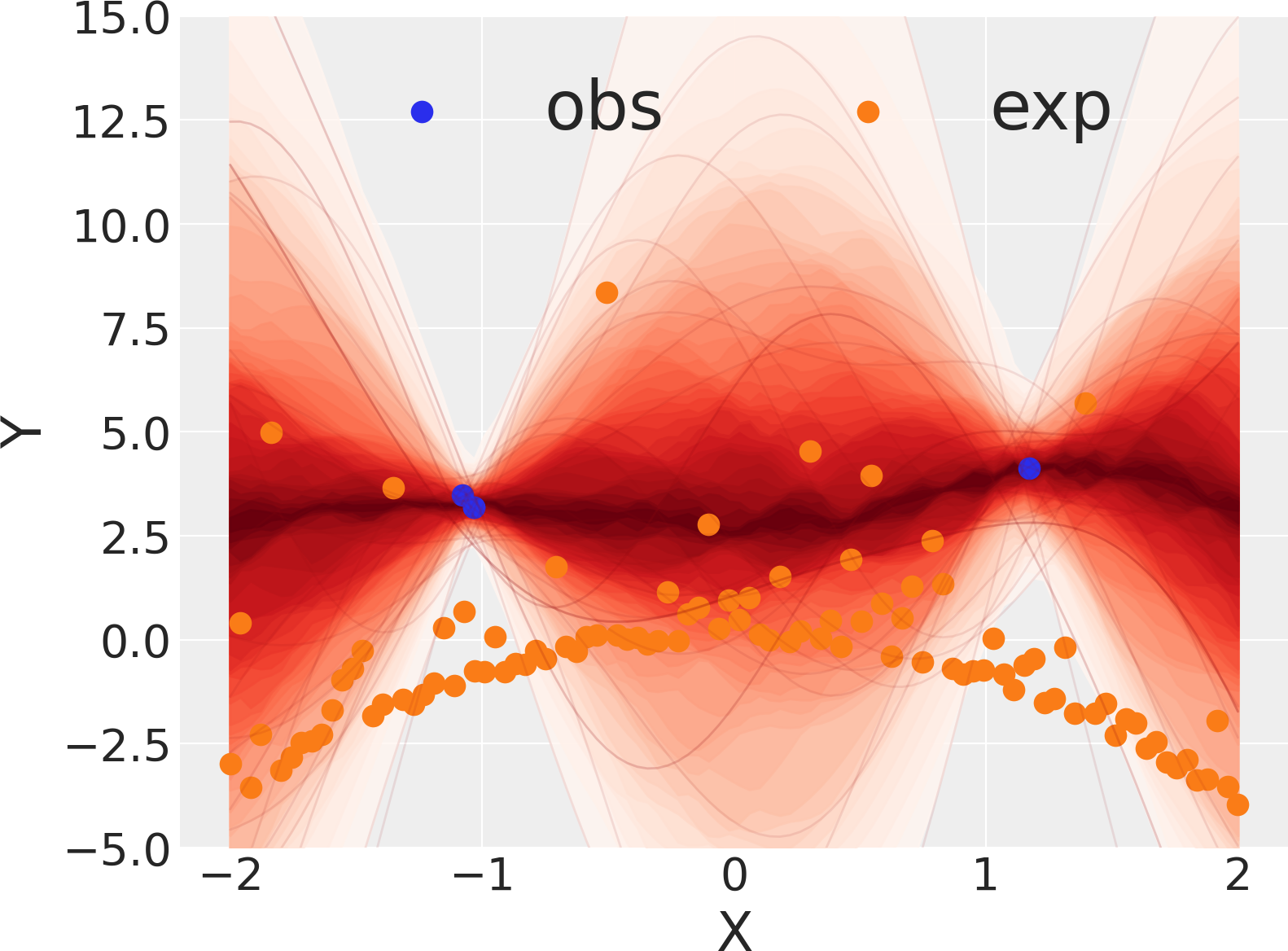}
		\caption{}
		\label{fig:_4_1_b}
	\end{subfigure}\hfill
        \begin{subfigure}{0.25\linewidth}\centering
            \setlength{\abovecaptionskip}{0pt}
		\includegraphics[width=\linewidth]{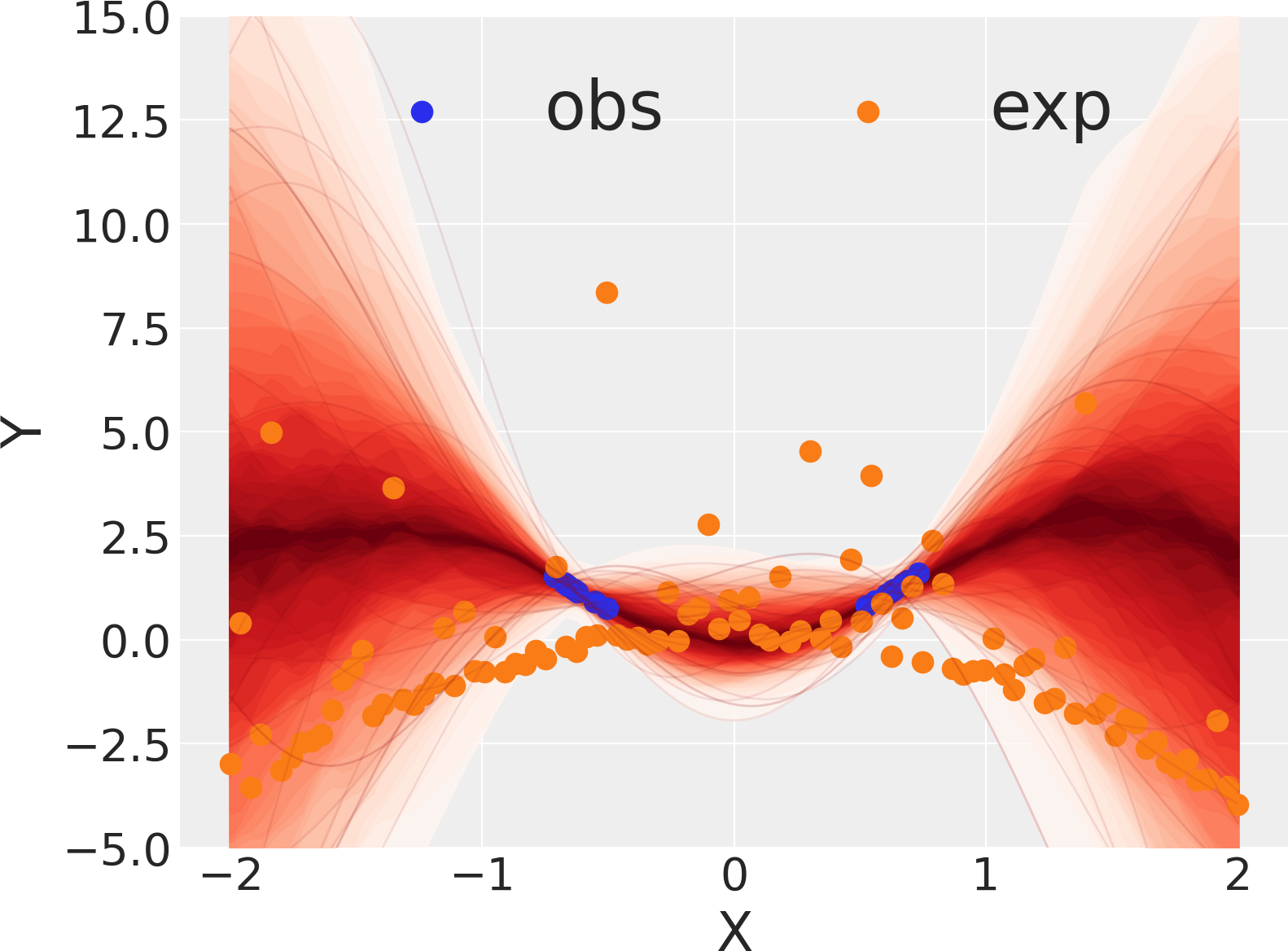}
		\caption{}
		\label{fig:_4_1_c}
	\end{subfigure}\hfill
        \begin{subfigure}{0.25\linewidth}\centering
            \setlength{\abovecaptionskip}{0pt}
		\includegraphics[width=\linewidth]{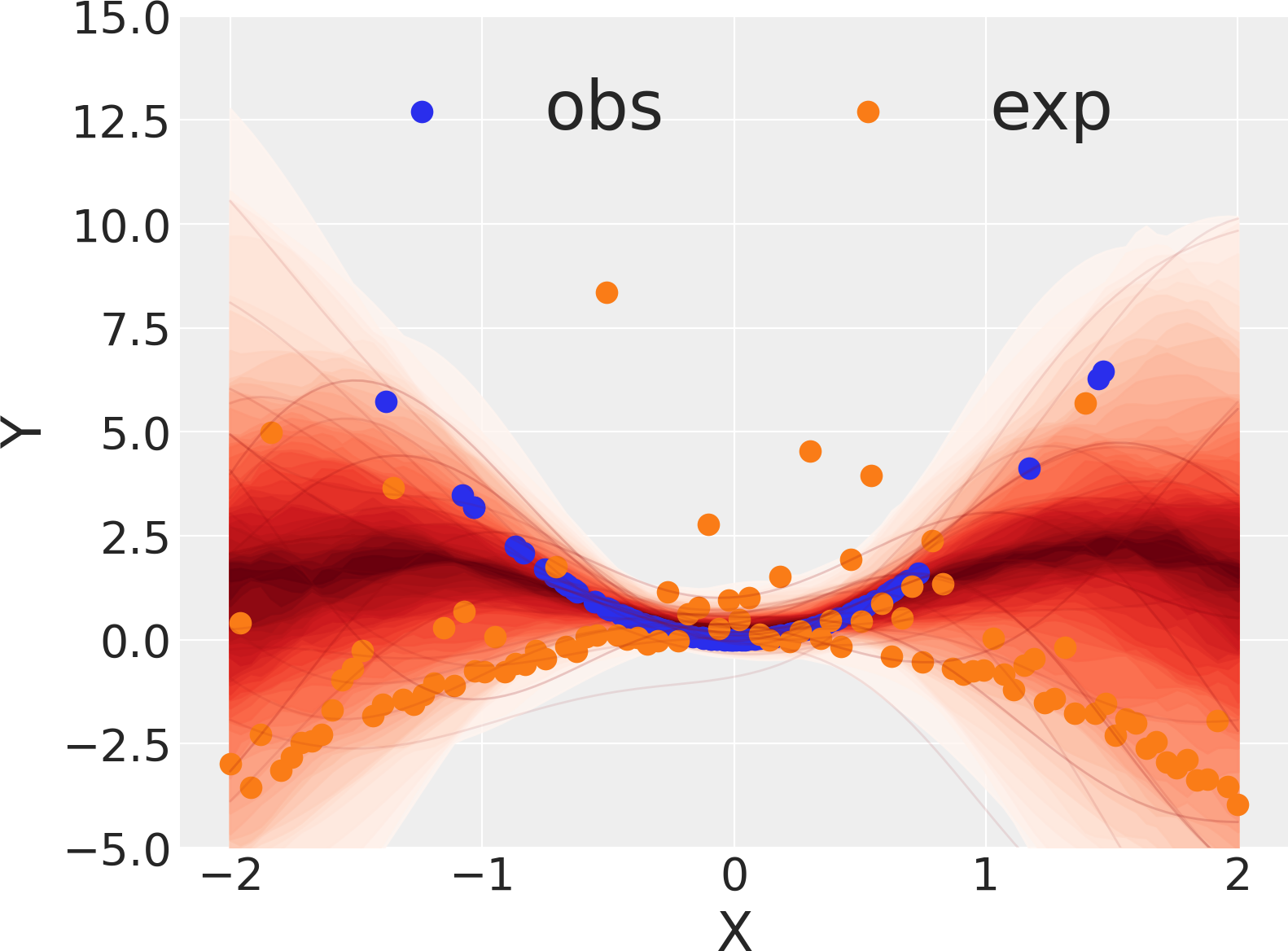}
		\caption{}
		\label{fig:_4_1_d}
	\end{subfigure}\hfill\null
	\caption{(\subref{fig:_4_1_a}) Stratified observed data based on the assigned functional types; (\subref{fig:_4_1_b}, \subref{fig:_4_1_c}) posteriors over selected canonical functions learned over stratified observations; (\subref{fig:_4_1_d}) the posterior over the treatment effect conditioning on observed data.}
	\label{fig:_4_1}
\end{figure*}
\subsection{Approximation with Continuous Functions}
For a canonical model in \cref{def:ccm}, the exogenous domain is discretized, and simple functions decide the values of action and reward. Note that every simple function could be approximated by a continuous function almost everywhere \citep[Ch.~1]{stein2009real}, leading to further simplification. Henceforth, we will refer to this simplified canonical model as the causal generative model. Formally,  
\begin{definition}\label{def:cgm}
    A causal generative model (CGM) $\1M$ is an SCM where exogenous variables $\*U = \{U\}$ are drawn from a \emph{discrete distribution} $P(U)$ over a finite domain $\{1, \dots, d\}$. For every unit $U = u$, values of covariates $\*Z$ and treatment $\*X$ are drawn from conditional distributions $P(\*Z \mid u)$ and $P(\*X \mid \*z, u)$ respectively; values of outcome $Y$ are decided by a \emph{continuous function} $h_u: \3R^{n+l} \mapsto \3R$ mapping from the treatment-context domain $\3R^{n+l}$ to reward domain $\3R$.
\end{definition}
Let $\2N$ denote the family of all causal generative models compatible with \Cref{fig:_mab_obs}. It follows immediately from \Cref{thm:approx} that the causal approximation property (\Cref{def:approx}) also holds for CGMs.
\begin{restatable}{theorem}{thmcgm}\label{thm:cgm}
    The set of CGM $\2N$ is dense in the set of all SCMs $ \2M$ compatible with \Cref{fig:_mab_obs}.
\end{restatable}
\Cref{thm:cgm} describes a novel family for generative models reproducing the observations and treatment effects in any SCM. The underlying population is stratified into a finite number of unknown reward functional types $h_u$ mapping from the domain of treatment $\*X$ and covariate $\*Z$ to reward $Y$. For every unit $U = u_i$, the nature reveal a partial context $\*z \sim P(\*Z \mid u_i)$, samples an observed action $\*x_i \sim P(\*X \mid \*z, u_i)$, performs an intervention following this action to unit $u_i$, and receives a subsequent reward $y_i \gets h_{u_i}(\*x_i)$. Fix a treatment $\*x \in \3R^n$ and a context $\*z \in \3R^l$. The treatment effect $\invE{Y \mid \*z}{\*x}$ is obtained by averaging the set of all potential reward functions $h_{u_i}$, weighted by the posterior exogenous probabilities $P(u_i \mid \*z)$. Formally, 
\begin{align}
    \invE{Y \mid \*z}{\*x} &= \sum_{u = 1}^d h_u(\*x, \*z)P(u \mid \*z) \label{eq:te}
\end{align}
The following example demonstrates the generative process of CGMs, and how it enables one to obtain a robust treatment effect evaluation from confounded observations.
\begin{example}\label{exp:_cgm}
    Consider the causal model $\1M$ described in \Cref{exp:_2}. Let $\3C_i(k) = \3U_i(k) \cup \3U_{2^{k+2}-i}(k)$, $i = 1, \dots, 2^{k+1}$ where $\3U_i(k)$ is defined in \Cref{eq:_4_3_1}. It follows from the discretization described in \Cref{exp:_2} that exogenous units $u\in \3C_i(k)$ where share the same reward function, i.e., $h_i^{(k)}(x) = -x^2 + \Parens{-2 + \frac{i - 1}{2^k}}^2$, for $u \in \3C_i(k), i = 1, \dots, 2^{k+1}$. \Cref{fig:_4_1_a} shows a partition of the observed data according to the type of associated function $h_i^{(k)}(x)$ for $k = 2$; samples generated by the same reward function are highlighted in the same color. For each $u$, the observed actions $x$ are drawn from a normal distribution $P(X|u)$. 
    
    We next apply standard Gaussian process regression to learn the canonical functions $h_i^{(k)}$ using the corresponding observed data; \Cref{fig:_4_1_b,fig:_4_1_c} show selected posteriors learned from this process. We also compute the posterior over the treatment effect $\invE{Y}{x}$ by averaging the learned posteriors over each canonical function $h_i^{(k)}(x)$, weighted by probabilities $P\Parens{U \in \3C_i(k)}$ assigned to each partition. The simulation result (\Cref{fig:_4_1_d}) shows that (1) the learned function $h_i^{(k)}(x)$ picks up behaviors of the real reward function $f_Y$ in the corresponding partition; (2) the learned posterior over $\invE{Y}{x}$ generalizes well to internveitonal data sampled from the ground-truth causal model. $\hfill \blacksquare$
\end{example}

\subsection{Causal Gaussian Processes}
\Cref{exp:_cgm} demonstrates that the membership to the same reward stratum $\3C_i(k)$ is a sufficient statistic satisfying the backdoor criterion \citep{pearl:95}. This allows one to estimate the treatment effect by adjustment on strata $\3C_i(k)$. If the propensity score $P(x|u) > 0$ has full coverage, one could eventually identify the ground-truth effect as the number of observed samples increases. On the other hand, the canonical functional stratification $\3C_i(k)$ is generally underdetermined by the confounded observations. To obtain a posterior over the target causal effect from confounded observations, it is thus essential to search over such potential partitioning of the exogenous domains and their functional assignments. We will next introduce a novel family of augmented Gaussian process models, called Causal GP, to carry out this learning procedure to obtain robust posteriors. 

More specifically, let $\*\theta$ denote the model parameters of the distributions (e.g., $P(\*u)$) and functions ($h_u$) in CGMs defined in \Cref{def:cgm}, and let $\theta_{\*x}(\*z)$ denote the values of a treatment effect $\invE{Y\mid \*z}{\*x}$ of action $\*X \gets \*x$ in context $\*Z = \*z$. Provided a set of finite samples $\1D = \Braces{(\*x_i, y_i, \*z_i)}_{i = 1}^N$, we are interested in computing the posterior distribution over a sequence of treatment effects $\1Q = \Braces{\theta_{\*x^*_i}(\*z^*_i)}_{i = 1}^{N^*}$ with treatment and context assignments $\Braces{(\*x^*_i, \*z^*_i)}_{i = 1}^{N^*}s$. The learner assumes access to a prior distribution over model parameters $\*\theta$ of candidate CGMs, the details of which are provided below. 

We assume the exogenous probabilities $P(u)$ are drawn from a truncated Dirichlet process, determining the total number of reward functional types and their assigned weights. A mental image of such a distribution follows a stick-breaking process \citep{sethuraman1994constructive}, which successively breaks off pieces from a unit-length stick with sizes proportional to random draws from a Beta distribution. Formally, for all $u = 1, \dots, d - 1$, 
\begin{align}
    P(u) = \rho_{u} \prod_{i = 1}^{u-1} (1 - \rho_i) \label{eq:dp}
\end{align} 
where $\rho_{u} \sim \texttt{Beta}\left (\alpha_{u}, \beta_{\*u} \right)$; hyperparameters $\alpha_{u}, \beta_{u} > 0$. Finally, we truncate this construction by setting $\rho_{u} = 1$.\footnote{In practice, the cardinality $d$ is bounded by the total number of observational samples and target queries $N + N^*$ to avoid overparametrization.} It follows from \Cref{eq:dp} that all probabilities $P(u)$ for $u > d$ are equal to zero.

Fix a unit $u$. Values of covariate $\*Z$ are drawn from a multivariate normal distribution $P(\*Z | u)$. That is, 
\begin{align}
    &\*z \sim \texttt{Normal}\left (\*\mu_{u}, \*\Sigma_{u} \right) &&\text{ for every } u = 1, \dots, d, 
\end{align}
where $\*\mu_u \in \3R^l$ is a $l$-dimensional mean vector and $\*\Sigma_{\*u} \in \3R^{l \times l}$ is a covariance matrix. For any context $\*z$, values of treatment $\*X$ are drawn from a multivariate normal distribution $P(\*X |\*z, u)$. That is, for any $u = 1, \dots, d$ and $\*z \in \3R^l$,
\begin{align}
    &\*x \sim \texttt{Normal}\left (\*\mu_{u}(\*z), \*\Sigma_{u}(\*z) \right), 
\end{align}
where $\*\mu_u(\*z) \in \3R^n$ and $\*\Sigma_{u}(\*z) \in \3R^{n\times n}$ are, respectively, the mean vector and covariance matrix. 

For any unit $u$, values of outcome $Y$ are decided by the function $h_u(\*x, \*z)$ mapping from the treatment-covariate domain to the primary outcome, the parameters of which are drawn from a Gaussian process. Formally, 
\begin{align}
      &\text{for every } u = 1, \dots, d, \;\; y \gets h_u(\*x, \*z), \\
      &h_u(\*x, \*z) \sim \texttt{GP}\left (m_{u}(\*x, \*z), k_{u}(\*x, \*z, \*x', \*z') \right), \label{eq:gp}
\end{align}
where $m_{u}(\*x, \*z) = \3E \Brackets{h_u(\*x, \*z)}$ is the expected function value given input $\*x, \*z$, and the covariance function $k_{u}(\*x,\*z, \*x', \*z')$ represents the correlation between function values at different input points $(\*x, \*z)$ and $(\*x', \*z')$, i.e., $k_{u}(\*x, \*z, \*x', \*z') = \3E \Brackets{\Parens{h_u(\*x, \*z) - m_{u}(\*x, \*z)}\Parens{h_u(\*x', \*z') - m_{u}(\*x', \*z')}}$. In our applications, we often set the prior mean function to $m_{u}(\*x, \*z) = \*0$ to avoid expensive posterior computations. The function $k_{u}$ is the kernel of the Gaussian process \citep{micchelli2006universal}. One very popular choice is the radial basis function kernel, which is defined as $k_{u}(\*x, \*z, \*x', \*z') = \sigma_{u}^2 \exp \Parens{- \frac{\lVert (\*x, \*z) - (\*x', \*z') \rVert^2}{2 \lambda_{u}^2}}$. Among these quantities, $\lambda_{u}$ and $\sigma_{u}^2$ are hyperparameters which can be varied to increase or reduce the correlations between points and the variability of the resulting function. \footnote{Notably, the choice of the kernel function $k_u$ injects the induction bias with regard to the smoothness and support of the underlying structural functions. Our simulation results show that the radial basis kernel performs well for partial identification tasks.}

Computing the exact derivation of the posterior distribution $P(\*\theta \mid \1D)$ is intractable. Variational inference (VI) approximates the posterior by constructing an optimization program \citep{blei2017variational}.  Consider a family of approximating densities of the latent model parameters $q(\*\theta ; \*\phi)$, parameterized by a vector $\*\phi \in \Phi$. VI finds the parameters $\*\theta^*$ that maximize the evidence lower bound (ELBO):
\begin{align}
    \1L(\*\phi) = \3E_{q(\*\theta)}\Brackets{\log{P(\*\theta, \1D)}} -\3E_{q(\*\theta)}\Brackets{\log{q(\*\theta ; \*\phi)}} \label{eq:elbo}
\end{align}
Among the above quantities, the first term is an expectation of the joint density under the approximation, and the second is the entropy of the variational density. The optimized $q(\*\theta ; \*\phi^*)$ then serves as an approximation to the posterior. Given model parameters drawn from the approximation, one could then obtain posterior samples of the target query $\theta_{\*x_k}(\*z_k)$ following the computation in \Cref{eq:te}.

Several approximation frameworks and algorithms have been proposed to solve the VI problem in \Cref{eq:elbo} \citep{kingma2013auto,blei2017variational,liu2016stein,kucukelbir2017automatic,roeder2017sticking}. In this paper, we mainly utilize the classic mean-field approximation with gradient-based optimization \citep{kucukelbir2017automatic}. However, other solvers are also generally applicable to our proposed causal generative models.

\begin{figure*}[t]
\centering
    \hfill
        \begin{subfigure}{0.25\linewidth}\centering
            \setlength{\abovecaptionskip}{0pt}
		\includegraphics[width=\linewidth]{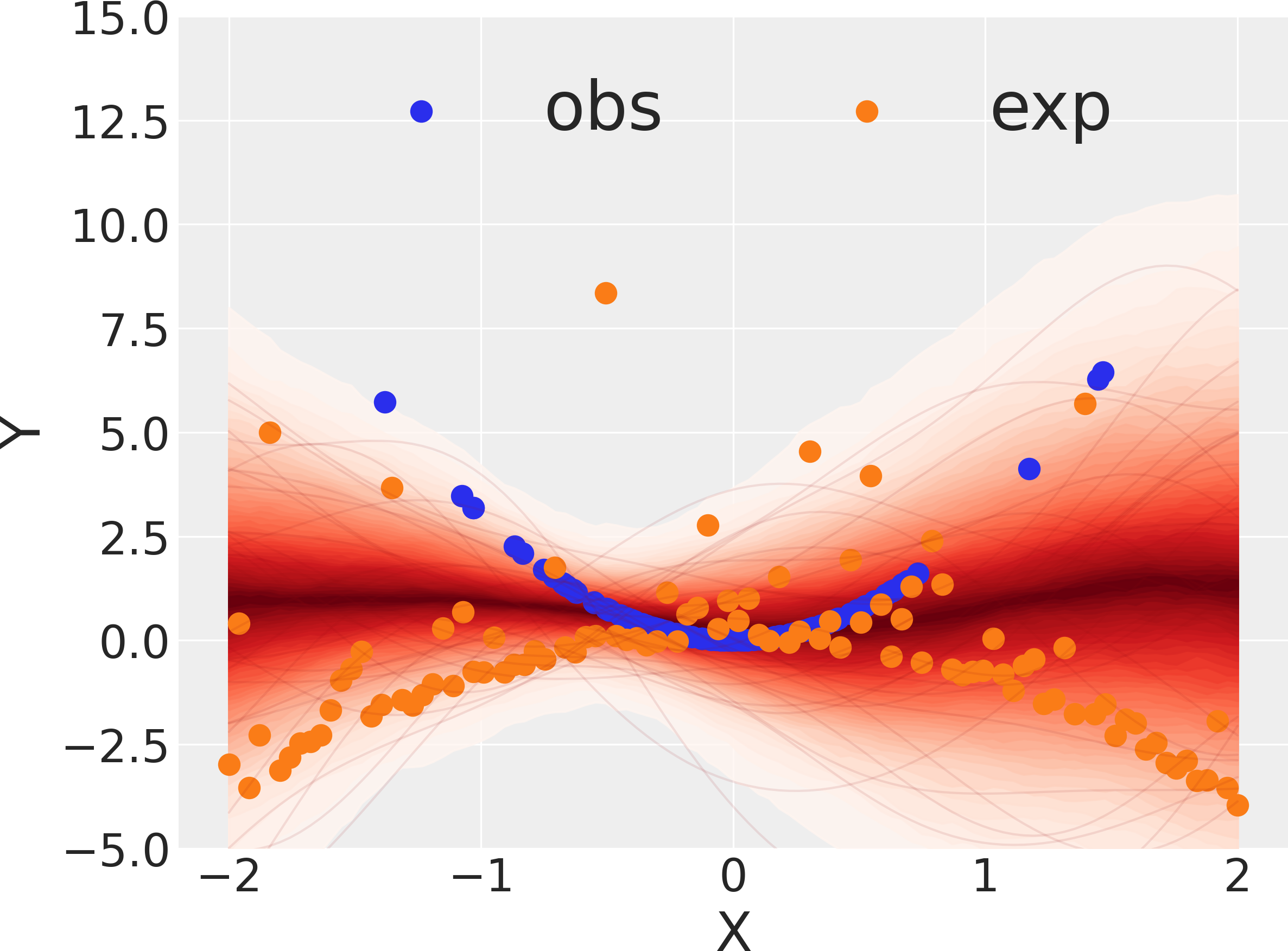}
		\caption{Polynomial}
		\label{fig:_5_1_a}
	\end{subfigure}\hfill
        \begin{subfigure}{0.25\linewidth}\centering
            \setlength{\abovecaptionskip}{0pt}
		\includegraphics[width=\linewidth]{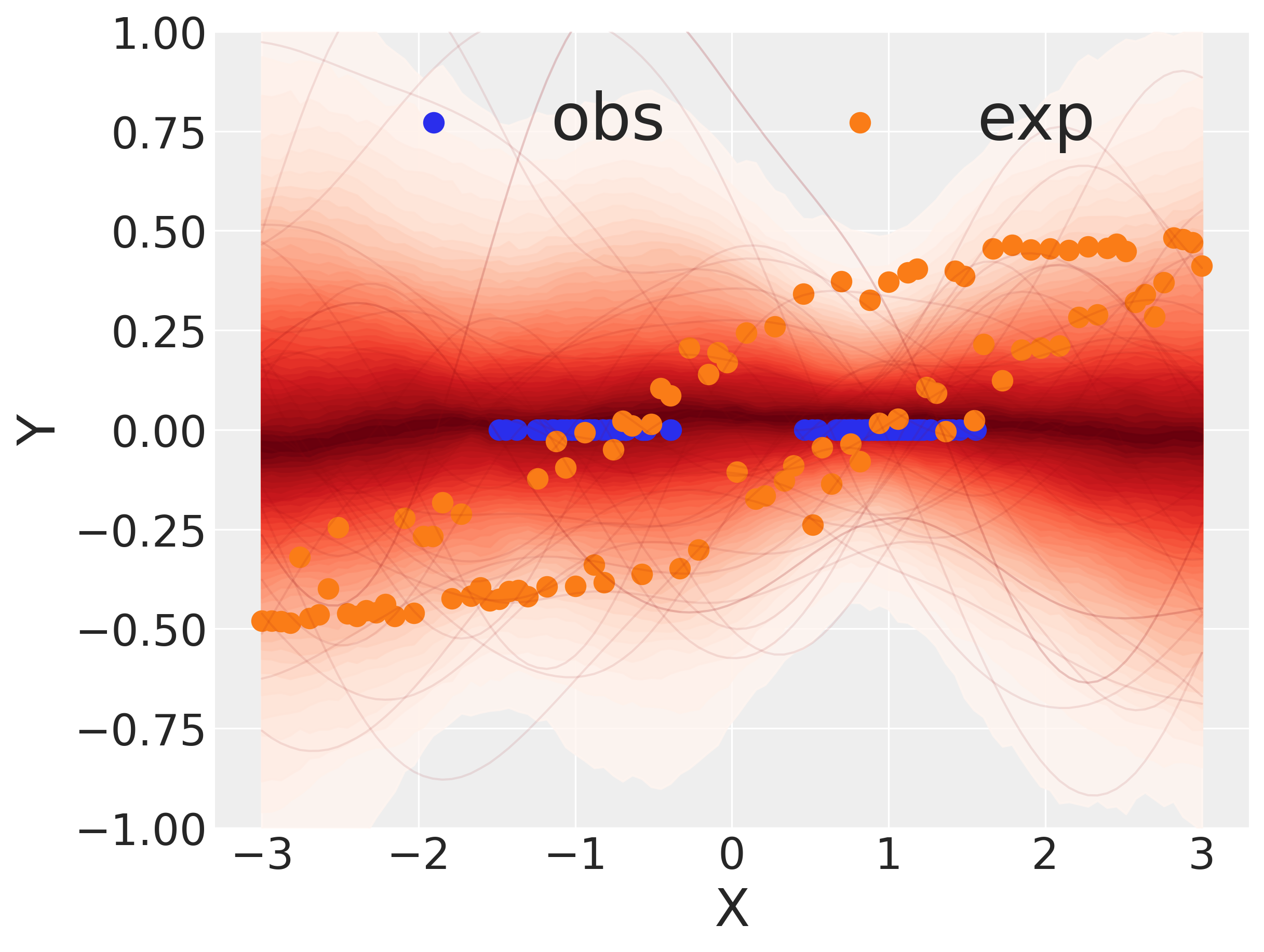}
		\caption{Logistic}
		\label{fig:_5_1_b}
	\end{subfigure}\hfill
        \begin{subfigure}{0.25\linewidth}\centering
            \setlength{\abovecaptionskip}{0pt}
		\includegraphics[width=\linewidth]{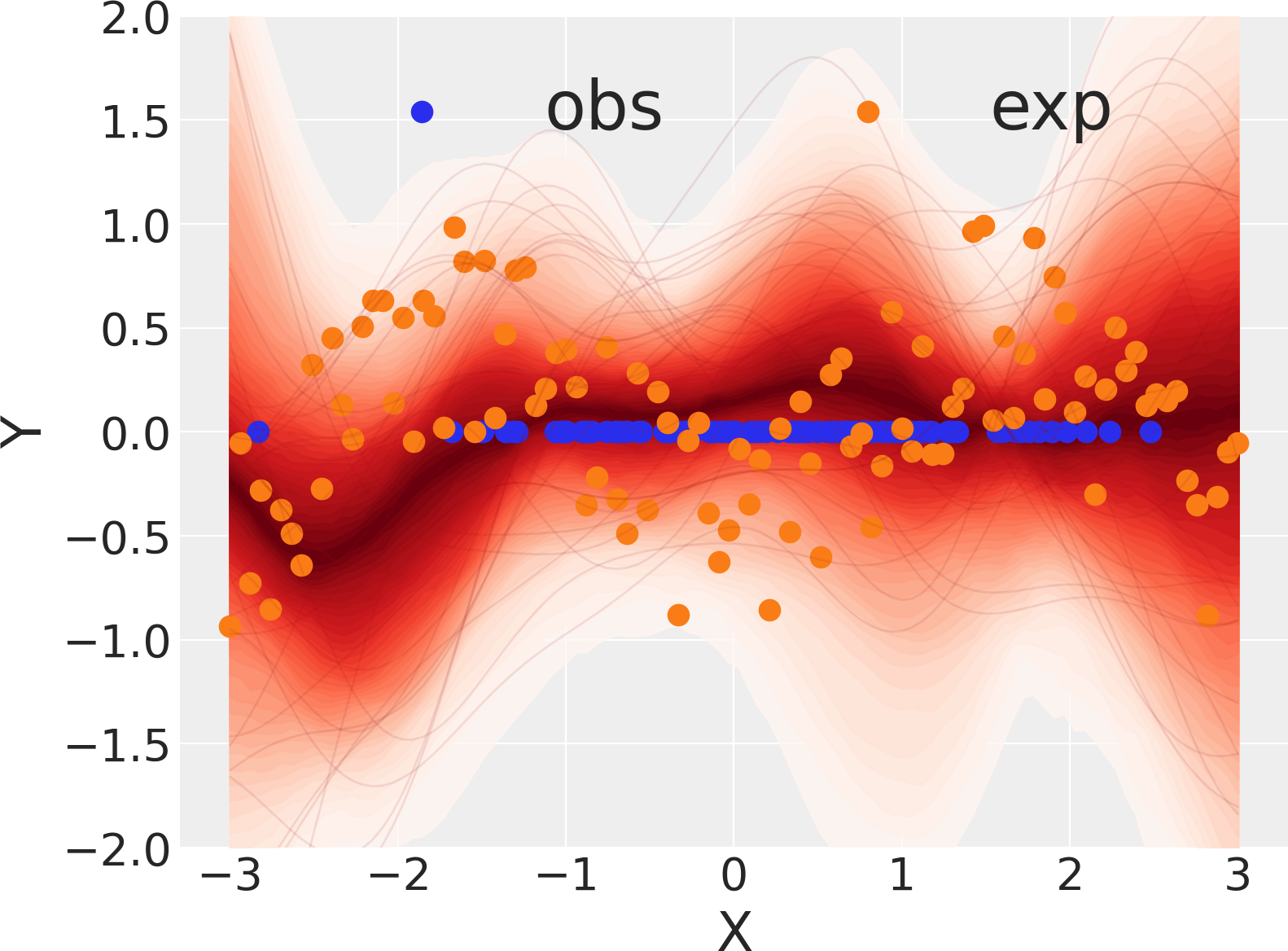}
		\caption{Phase}
		\label{fig:_5_1_c}
	\end{subfigure}\hfill
    \begin{subfigure}{0.25\linewidth}\centering
            \setlength{\abovecaptionskip}{0pt}
		\includegraphics[width=\linewidth]{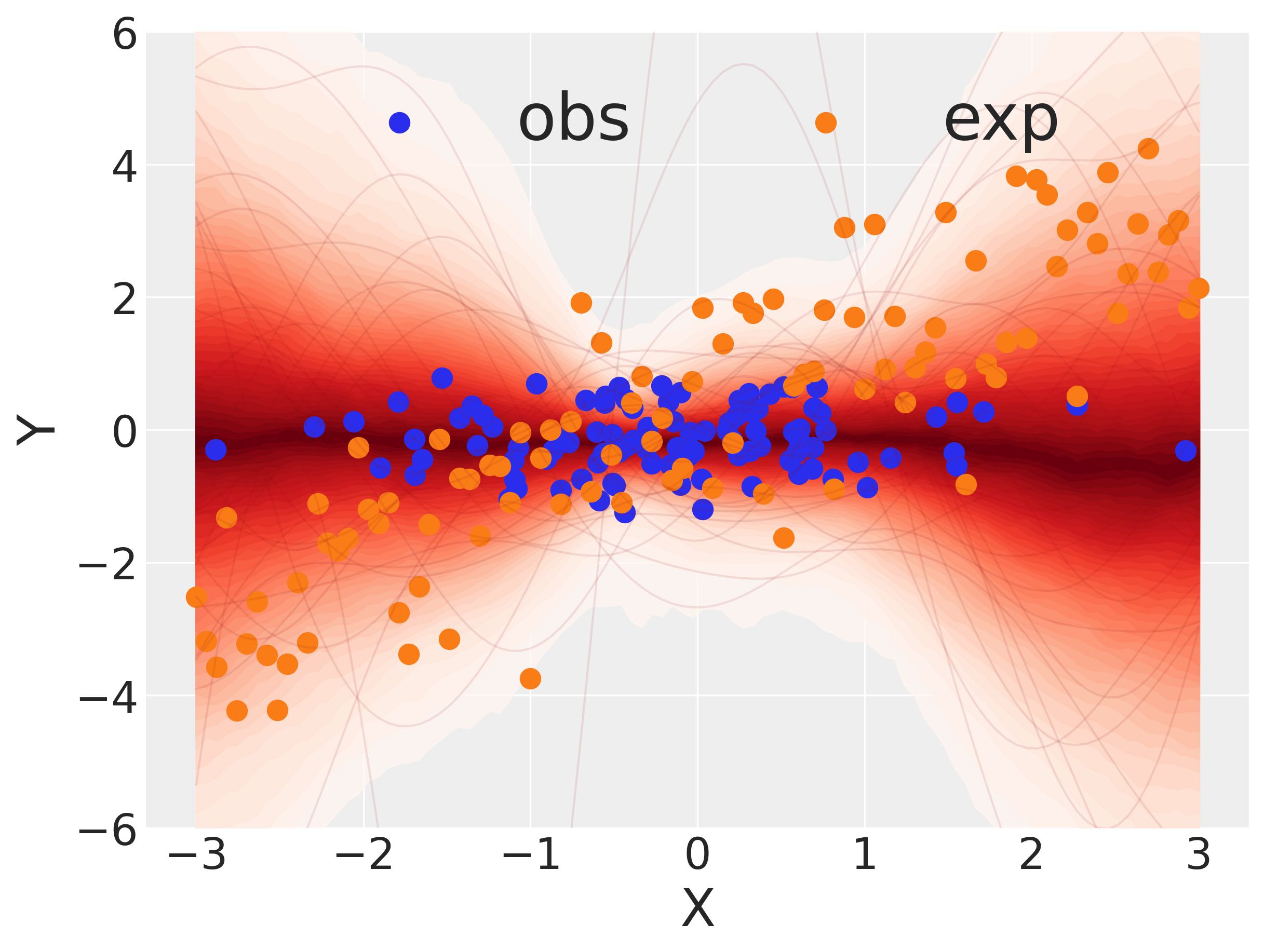}
		\caption{Linear}
		\label{fig:_5_1_d}
	\end{subfigure}\hfill\null
	\caption{Simulations comparing the derived posterior approximations over various reward functions using our proposed causal generative model. These functions include: (\subref{fig:_5_1_a}) polynomial; (\subref{fig:_5_1_b}) logistic; (\subref{fig:_5_1_c}) phase; and (\subref{fig:_5_1_d}) linear.}
	\label{fig:_5_1}
\end{figure*}

\begin{table*}[t]
   \centering
   \caption{Simulation results compare the negative log predictive likelihood of posterior distributions derived from standard GP and causally-enhanced GP models. A lower score indicates a better model fit.}
   \label{tab:gp_table}
   \begin{tabular}{|c|C{2cm}C{2cm}C{2cm}C{2cm}C{2cm}|}
     \toprule
     Model & Polynomial $\downarrow$ & Logistic $\downarrow$ & Phase $\downarrow$ & Linear$ \downarrow$ & IST $\downarrow$ \\
     \midrule
     Standard GP & 812.25328944  & 212.90048944 & 186.74978944 &  873.15343639 & 1981.7912937 \\
     Causal GP & \textbf{370.78863053} & \textbf{204.75931926} & \textbf{176.92091926} & \textbf{463.48699235} &  \textbf{34.06876521}\\
     \bottomrule
   \end{tabular}
\end{table*}

\section{Simulations and Experiments}
We demonstrate our algorithms using causal models with various reward functions and real-world data collected from the International Stroke Trial (IST; \citet{carolei1997international}). For all experiments, we evaluate the learned posteriors $P(\*\theta \mid \1D)$ using the negative log predictive likelihood (NLPL) $\1S$ over experimental data $\1D^* = \{(\*x_i^*, y_i^*, \*z_i^*)\}_{i=1}^{N^*}$ drawn from the actual interventional distribution $P(Y, \*Z \mid \doo(\*x))$. That is,
\begin{align}
    \1S = - \log{\int_{\*\theta} \prod_{i = 1}^{N^*} P(y_i^* \mid \doo(\*x_i^*), \*z^*_i; \*\theta) P(\*\theta \mid \1D) d\*\theta} \label{eq:_nlpl}
\end{align}
The above score $\1S$ is also referred to as Bayes Free Energy \citep{watanabe2013widely}, which can be approximated using the Bayesian Information Criterion (BIC; \citep{schwarz1978estimating}) in regular statistical models (e.g., GP).

We report the NLPL score $\1S$ comparing the standard GP and the proposed Causal GP in \Cref{tab:gp_table}. We also plot the posteriors over treatment effects learned by CGP for synthetic models, shown in \Cref{fig:_5_1}. In all experiments, the actual interventional data collected from randomized trials deviate significantly from the observational data due to unobserved confounding. Our analysis reveals that causal GP consistently achieves lower NLPL scores across all experiments than standard GP, suggesting better generalization to actual experimental data. Additionally, the posterior learned by Causal GP (highlighted in \textcolor{red}{red} in \Cref{fig:_5_1}) contains the actual experimental data (\textcolor{orange}{orange}). On the other hand, standard GP fails to learn the actual treatment effects since it consistently overfits the confounded observations (\textcolor{blue}{blue} points in \Cref{fig:gp}). For more details on the experimental setup, we refer readers to \Cref{appendix:b}.

\textbf{Polynomial Function.} Consider again the causal model described in \Cref{exp:_1} where the reward function is a polynomial function $Y \gets -X^2 + U^2$. We apply our proposed causal GP to derive posteriors over the treatment effects $\invE{Y}{x}$ for latent cardinality $d = 30$. The posterior distributions obtained by the standard GP and causal GP models are shown in \Cref{fig:gp_confounded,fig:_5_1_a} respectively.

\textbf{Logistic Function.} Consider a causal model where values of action $X$ and reward $Y$ are given by $X \gets E \times U$ and $Y \gets 1 / (1 + e^{-X + E\times U})$ respectively; $U$ is an unobserved variable drawn from a standard normal distribution, and $E$ is a uniformly drawn over a binary domain $\{-1, 1\}$. \Cref{fig:gp_confounded_logistic,fig:_5_1_b} show the posterior distributions obtained by the standard GP and causal GP from confounded observations, respectively. The observational data is highlighted in \textcolor{blue}{blue}, and the experimental data in \textcolor{orange}{orange}.

\textbf{Phase Function.} Consider a causal model where values of action $X$ and reward $Y$ are given by $X \gets U$ and $Y \gets \sin(X)^2 + \cos(U)^2 - 1$ respectively; $U$ is an unobserved variable drawn from a standard normal distribution. We show in \Cref{fig:gp_confounded_phase,fig:_5_1_c} the posterior distributions obtained by the standard GP and causal GP from confounded observations, respectively.

\textbf{Linear Function.} Consider a causal model where values of action $X$ and reward $Y$ are given by $X \gets -U$ and $Y \gets X + U + E$ respectively; $U$ is an unobserved variable drawn from a standard normal distribution, and $E$ is an independent noise drawn from a normal distribution $\texttt{Normal}(0, 0.5)$. \Cref{fig:gp_confounded_linear,fig:_5_1_d} shows the posterior distributions obtained by the standard GP and causal GP from confounded observations, respectively.

\textbf{IST.} The International Stroke Trial \citep{carolei1997international} is a randomized study of over $19,000$ patients that established aspirin as a treatment for acute ischemic stroke. The database includes patient age, gender, conscious state at randomization, and systolic blood pressure at randomization, among other variables. We are interested in evaluating the causal effect of a patient's blood pressure $X$ on their predicted probability of death $Y$ at 14 days after admission. We conscious state $U$ as unobserved confounding, and the age $Z$ as the observed covariate. To simulate confounding bias, we create observational data by selecting the dataset based on a patient's conscious state and age.

\Cref{tab:gp_table} shows the negative log predictive likelihood score $\1S$ for posteriors obtained by both Causal GP and standard GP. The lower the score is, the better the posterior fits the actual experimental data. As expected, standard GP fails to generalize as it overfits confounded observations.
\vspace{-0.1in}
\section{Conclusions}
This paper addresses the challenge of evaluating causal effects from confounded observational data in complex domains, focusing on the canonical bandit model with basic temporal ordering among action and reward variables. We introduce a new family of causal generative models with a finite number of latent states, which can accurately approximate the observational distribution and treatment effects in any causal model. Using this framework, we present a novel Causal Gaussian Process model that automates latent-space discretization, enabling the derivation of posterior distributions over target causal effects despite confounding. Simulation results show that these posteriors are resilient to unmeasured confounding and generalize well to actual experimental data collected from randomized trials.
%Future work will aim to develop models for time series data and integrate additional qualitative causal knowledge.

%\section*{Reproducibility Statement}
%The complete proof of all theoretical results presented in this paper, including \Cref{thm:approx} and \Cref{prop:eqx,prop:equ,prop:cgm}, can be found in Appendix \ref{appendix:a}. Detailed descriptions of the experimental setup are provided in Appendix \ref{appendix:b}. All appendices are included as part of the supplementary material following the "References" section. Please note that all experiments are synthetic and do not involve the introduction of any new assets.

% References
\bibliography{book,add_ref}

\newpage

\onecolumn

\title{Causal Gaussian Processes for Robust Treatment Effect Evaluation with Unobserved Confounding (Supplementary Material)}
\maketitle

\appendix
\crefalias{section}{appendix}

\section{Proofs} \label{appendix:a}
This section will provide proof of all the theoretical results in the paper. We will first introduce some propositions that are necessary for the construction of causal canonical models (\Cref{prop:eqx,prop:equ}). We then put things together and formally show the causal approximation property of CCMs (\Cref{thm:approx}).

\textbf{Discretizing State-Action Space.} We first introduce the necessary tools to construct the causal canonical model. First, we will describe a simple function parametrization of the treatment effect that allows us to partition the state-action space into finite equivalence classes.
\begin{definition}\label{def:simple_e}(Simple Treatment effects)
    Let $\1M$ be an SCM with action $\*X$, reward $Y$ and covariate $\*Z$. The simple treatment effect of action $\*X$ on reward $Y$ conditioning on $\*Z$ is a function $\widehat{\3E}_{\*x}[Y \mid \*z]: \3R^{n+l} \mapsto \3R$ of the form $\widehat{\3E}_{\*x}[Y \mid \*z] = \sum_{i = 1}^N \invE{Y \mid \*z_i}{\*x_i} \I_{\3X_i}(\*x)\I_{\3Z_i}(\*z)$, where $\3X_1, \dots, \3X_N$ and $\3Z_1, \dots, \3Z_N$ are disjoint subsets in $\3R^n$ and $\1R^l$ respectively; and $\*x_i$ and $\*z_i$ are, respectively a realization in $\3X_i$ and $\3Z_i$ for every $i = 1, \dots, N$.
\end{definition}
The next proposition ensures that the simple treatment effect function in \Cref{def:simple_e} can effectively approximate measurable treatment effects for all actions $\*x \in \3R^n$ in an arbitrary causal model.
\begin{restatable}[Equivalence Classes of State and Action]{proposition}{propeqx}\label{prop:eqx}
    Let $\1M$ be an SCM with action $\*X$, reward $Y$ and covariate $\*Z$. Then for every $\epsilon > 0$, there exists a simple causal effect $\widehat{\3E}_{\*x}[Y \mid \*z]$ such that $\int_{\3R^n} \left \lvert \3E_{\*x} \Brackets{Y \mid \*z} - \widehat{\3E}_{\*x} \Brackets{Y \mid \*z}  \right \rvert d\*x d\*z < \epsilon$. Henceforth, we will consistently refer to the finite sequence of subsets $\3X_1, \dots, \3X_k$ (or realizations $\*x_1, \dots, \*x_k$) associated with $\widehat{\3E}_{\*x}[Y \mid \*z]$ as \emph{equivalence classes of action}.
\end{restatable}
\begin{proof}
The proof follows from \citep[Prop.~3.44]{axler2020measure}. Specifically, let fix $\invE{Y \mid \*z}{x}^+$ and $\invE{Y \mid \*z}{x}^-$ denote the positive and negative parts of the treatment effects $\invE{Y \mid \*z}{x}$, respectively. For any $\epsilon > 0$ there exist simple functions $\widehat{\3E}_{\*x}[Y \mid \*z]^{(1)}$ and $\widehat{\3E}_{\*x}[Y \mid \*z]^{(2)}$ such that
\begin{align}
    &\int_{\3R^{n+l}} \left \lvert \3E_{\*x}^+ \Brackets{Y \mid \*z} - \widehat{\3E}_{\*x}^{(1)} \Brackets{Y \mid \*z} \right \rvert d\*x d \*z < \frac{\epsilon}{2}, &&\int_{\3R^{n+l}} \left \lvert \3E_{\*x}^- \Brackets{Y \mid \*z} - \widehat{\3E}_{\*x}^{(2)} \Brackets{Y \mid \*z}  \right \rvert d\*x d\*z < \frac{\epsilon}{2}
\end{align}
The existence of $\widehat{\3E}_{\*x}[Y \mid \*z]^{(1)}$ and $\widehat{\3E}_{\*x}[Y \mid \*z]^{(2)}$ follows from \citep[Prop.~3.9]{axler2020measure}. Let $\widehat{\3E}_{\*x}[Y \mid \*z] = \widehat{\3E}_{\*x}[Y \mid \*z]^{(1)} - \widehat{\3E}_{\*x}[Y \mid \*z]^{(2)}$. Then $\widehat{\3E}_{\*x}[Y \mid \*z]$ is a simple function and
\begin{align}
    &\int_{\3R^{n+l}} \left \lvert  \3E_{\*x} \Brackets{Y \mid \*z} - \widehat{\3E}_{\*x} \Brackets{Y \mid \*z}  \right \rvert d\*x d\*z \\
    &= \int_{\3R^{n+l}} \left \lvert \3E_{\*x}^+ \Brackets{Y \mid \*z} - \widehat{\3E}_{\*x}^{(1)}\Brackets{Y \mid \*z} + \3E_{\*x}^- \Brackets{Y \mid \*z} - \widehat{\3E}_{\*x}^{(2)} \Brackets{Y}  \right \rvert d\*xd\*z \\
    &\leq \int_{\3R^[n+l]} \left \lvert \3E_{\*x} \Brackets{Y \mid \*z}^+ - \widehat{\3E}_{\*x} \Brackets{Y \mid \*z}^{(1)} \right \rvert d\*x d\*z + \int_{\3R^{n+l}} \left \lvert \3E_{\*x} \Brackets{Y \mid \*z}^- - \widehat{\3E}_{\*x} \Brackets{Y \mid \*z}^{(2)} \right \rvert d\*x d\*z
\end{align}
Simplifying the above equation gives $\int_{\3R^{n+l}} \left \lvert \3E_{\*x} \Brackets{Y \mid \*z} - \widehat{\3E}_{\*x} \Brackets{Y \mid \*z}  \right \rvert d\*x d\*z \leq \epsilon$. This proves the statement.
\end{proof}
\Cref{prop:eqx} says that for any causal model $\1M$, there must exist a simple treatment effect function that is capable of approximating the original treatment effect in $\1M$ with arbitrary precision. Moreover, the simple treatment effect in \Cref{def:simple_e} only takes values of the original treatment effect at a finite set of realized state-action pairs $(\*x_1, \*z_1), \dots, (\*x_N, \*z_N)$. For any other realizations $(\*x, \*z)$, its treatment effect will be approximated with that of the action $\*x_i$ given context $\*z_i$ such that both pairs $(\*x, \*z)$ and $(\*x_i, \*z_i)$ belong to the same product of equivalence class $\3X_i \times \*Z_i$. It is thus sufficient to only model the treatment effects induced by finite interventions $\doo(\*x_1), \dots, \doo(\*x_N)$. 

\paragraph{Discretizing Exogenous Domain.}So far we have decomposed the action domain $\3R^n$ and reduced it to a finite set of equivalence classes $\*x_1, \dots, \*x_N$. This means that to effectively approximate the observational distribution and interventional effect in the ground-truth model $\1M$, it suffices to model a finite number of potential outcome variables $\*X, Y$ and $Y_{\*x_1}, \dots, Y_{\*x_N}$. \footnote{We will refer to observed variables $\*X, Y, \*Z$ as potential outcomes of the natural intervention $\doo(\emptyset)$.}\footnote{Since $\*Z$ is a non-descendant of $\*X$, the potential outcome $\*Z_{\*x}(\*u) = \*Z$ for any $\*x, \*u$ \citep{galles:pea97b}. } We will next introduce a family of simple functions mapping from the exogenous domains $\3R^m$ to the action $\3R^n$, the outcome domain $\3R$, or the covariate domain $\3R^l$ to approximate this finite set of potential outcomes.
\begin{restatable}[Equivalence Classes of Exogenous]{proposition}{propequ}\label{prop:equ}
    Let $\1M$ be an SCM with action $\*X$, reward $Y$, and covariate $\*Z$. Let $\3X$ be a finite set of realizations $\*x \in \3R^n$. Then for every $\epsilon, \delta > 0$, there exist simple potential outcomes $\widehat{\*X}, \widehat{Y}$, $\widehat{\*Z}$ and $\widehat{Y}_{\*x}$, for all $\*x \in \3X$, that converge to their corresponding potential outcomes $\*X, Y, \*Z$ and $Y_{\*x}$, for all $\*x \in \3X$, almost everywhere. That is,
    \begin{align}
        &P\Parens{\Braces{\*u \in \3R^m :\left \lvert \widehat{\*X}(\*u) - \*X(\*u) \right \rvert + \left \lvert \widehat{Y}(\*u) - Y(\*u) \right \rvert + \left \lvert \widehat{\*Z}(\*u) - \*Z(\*u) \right \rvert + \sum_{\*x \in \3X} \left \lvert \widehat{Y}_{\*x}(\*u) - Y_{\*x}(\*u) \right \rvert > \epsilon} } < \delta,
    \end{align}
    Henceforth, we will consistently refer to the finite sequence of subsets $\3U_1, \dots, \3U_k$ associated with $\widehat{\*X}, \widehat{Y}, \widehat{\*Z}$ and $\widehat{Y}_{\*x}$, for all $\*x \in \3X$, as \emph{equivalence classes of exogenous}.
\end{restatable}
\begin{proof}
    We first consider the approximation for a potential outcome variable $Y_{\*x}(\*u)$. This approximation procedure extends immediately for a combination of observed variables $\*X(\*u), Y(\*u), \*Z(\*u)$ and multiple potential outcomes $Y_{\*x_1}(\*u), \dots, Y_{\*x_k}(\*u)$. 

    The idea of the proof follows \citep[Prop.~2.89]{axler2020measure}. For each $k \in \3Z^+$ and $n \in \3Z$, the interval $[n, n + 1)$ is divided into $2^k$ equally sized half-open subintervals. If $Y_{\*x}(\*u) \in [0, k]$, we define $\widehat{Y}^{(k)}_{\*x}(\*u)$ to be the left endpoint of the subinterval into which $Y_{\*x}(\*u)$ falls; if $Y_{\*x}(\*u) \in [-k, 0)$, we define $\widehat{Y}^{(k)}_{\*x}(\*u)$ to be the right endpoint of the subinterval into which $Y_{\*x}(\*u)$ falls; and if $\lvert Y_{\*x}(\*u) \rvert > k$, we define $\widehat{Y}^{(k)}_{\*x}(\*u)$ to be $\pm k$. Specifically, let
    \begin{align}
        \widehat{Y}^{(k)}_{\*x}(\*u) =
        \begin{dcases}
            \frac{m}{2^k} &\mbox{if $0 \leq Y_{\*x}(\*u) \leq k$ and $m \in \3Z$ is such that } Y_{\*x}(\*u) \in \left [ \frac{m}{2^k}, \frac{m+1}{2^k}\right )\\
            \frac{m+1}{2^k} &\mbox{if $-k \leq Y_{\*x}(\*u) \leq 0$ and $m \in \3Z$ is such that } Y_{\*x}(\*u) \in \left [ \frac{m}{2^k}, \frac{m+1}{2^k}\right )\\
            k &\mbox{if $Y_{\*x}(\*u) > k$}\\
            -k &\mbox{if $Y_{\*x}(\*u) < -k$}
        \end{dcases}
    \end{align}
    The definition of $\widehat{Y}^{(k)}_{\*x}(\*u)$ implies that for all $\*u \in \3R^m$ such that $Y_{\*x}(\*u) \in [-k, k]$, $ \left \lvert \widehat{Y}^{(k)}_{\*x}(\*u) - Y_{\*x}(\*u) \right \rvert \leq 1/2^k$. Increase the value of $k \in \3Z^+$ such that $1/ 2^k < \epsilon$ and $P\Parens{Y_{\*x}(\*U) \in [-k, k]} < \delta$, which proves the statement for a potential outcome $Y_{\*x}(\*U)$. We could repeat the same discretization procedure for the observed action $\*X(\*u)$, observed reward $Y(\*u)$, covariate $\*Z(\*u)$, and every potential outcome $Y_{\*x}(\*U)$, $\*x \in \3X$. Taking intersections of these partitionings over the exogenous domain $\*U \in \3R^m$ completes the proof. 
\end{proof}

\textbf{Constructing Canonical Models.} We are now ready to put things together and provide a formal justification for the approximation property of canonical models. Specifically, given finite equivalence classes over the action and exogenous domain, one could obtain a canonical model $\1N$ that approximates the treatment effect $\invE{Y}{\*x}$ in the original causal model $\1M$ with arbitrary precision. This approximation is supported by \Cref{prop:eqx,prop:equ}.

What remains is to ensure that the constructed model $\1N$ is consistent with the ground-truth model $\1M$ in terms of the observational distribution $P(\*X, Y, \*Z)$. We will achieve this property by enumerating through the finite exogenous equivalence classes $\*u_1, \dots, \*u_N$; for every $\*u_i$, add the approximate observed action $\widehat{\*x}_i \gets \widehat{\*X}(\*u_i)$ to the equivalence classes of action $\3X$. 
For every newly added action $\widehat{\*x}_i$, we set its simple potential outcome $\widehat{Y}_{\widehat{\*x}'_i}(\*u_i)$ as the observed outcome $\widehat{Y}(\*u_i)$ in the constructed model $\1N$; for every unit such that $\widehat{\*X}(\*u_j) \neq \widehat{\*x}_i$, we will find the original action equivalence class $\3X_i$ such that $\widehat{\*x}_i \in \3X_i$, and set the simple potential outcome $\widehat{Y}_{\widehat{\*x}'_i}(\*u_j) = \widehat{Y}_{\*x_i}(\*u_j)$. Thanks to the construction of simple functions, this adjustment ensures the canonical model $\1N$'s consistency on the observational distribution $P(\*X, Y, \*Z)$ while maintaining its consistency on the treatment effect $\invE{Y \mid \*z}{\*x}$.
\thmapprox*
\begin{proof}
    Given finite equivalence classes over the action and exogenous domain, one could obtain a canonical model $\1N$ that approximates the treatment effect $\invE{Y \mid \*z}{\*x}$ in the original causal model $\1M$ with arbitrary precision. Specifically, let $\3U^c(k)$ be the set of exogenous values $\*u$ defined in \Cref{prop:equ}, and let $\3U(k) = \3R^m \backslash \3U^c(k)$. The $L_1$ error between the treatment effects defined in the SCM $\1M$ and the CCM $\1N$ could be written as:
    \begin{align}
        &\int_{\3R^{n+l}} \left \lvert \invE{Y \mid \*z; \1M}{\*x} - \invE{Y \mid \*z; \1N}{\*x}  \right \rvert d\*x d\*z \\
        &= \underbrace{\int_{\3U(k)}  \int_{\3R^{n+l}} \left \lvert \invE{Y(\*u) \mid \*z; \1M}{\*x} - \invE{Y(\*u) \mid \*z; \1N}{\*x}  \right \rvert d\*x d\*z dP(\*u)}_{\text{Term 1}}\\
        &+\underbrace{\int_{\3U^c(k)} \int_{\3R^{n+l}} \left \lvert  \invE{Y(\*u) \mid \*z; \1M}{\*x} - \invE{Y(\*u) \mid \*z; \1N}{\*x}  \right \rvert d\*x d\*z dP(\*u)}_{\text{Term 2}}
    \end{align}
    By the construction of set $\3U^c(k)$, Term 2 in the above equation gets increasingly smaller as the discretization granularity $k \in \3Z^+$ increases. Similarly, one could minimize the $L_1$ error with an arbitrary accuracy by increasing the parameter $k$ and the discretization granularity over the action domain $\*x \in \3R^n$. These approximation properties are supported by \Cref{prop:eqx,prop:equ}.

    What remains is to ensure that the constructed model $\1N$ is consistent with the ground-truth model $\1M$ in terms of the observational distribution $P(\*X, Y, \*Z)$. We will achieve this property by enumerating through the finite exogenous equivalence classes $\*u_1, \dots, \*u_N$; for every $\*u_i$, add the approximate observed action $\widehat{\*x}_i \gets \widehat{\*X}(\*u_i)$ to the equivalence classes of action $\3X$. For every newly added action $\widehat{\*x}_i$, we set its simple potential outcome $\widehat{Y}_{\widehat{\*x}'_i}(\*u_i)$ as the observed outcome $\widehat{Y}(\*u_i)$ in the constructed model $\1N$; for every unit such that $\widehat{\*X}(\*u_j) \neq \widehat{\*x}_i$, we will find the original action equivalence class $\3X_i$ such that $\widehat{\*x}_i \in \3X_i$, and set the simple potential outcome $\widehat{Y}_{\widehat{\*x}'_i}(\*u_j) = \widehat{Y}_{\*x_i}(\*u_j)$. Since this construction procedure only adds a finite number of points to the equivalence class of the action domain, these points have Lebesgue measure zero and do not affect the $L_1$ approximation error.
\end{proof}

\thmcgm*
\begin{proof}
\Cref{thm:approx} implies that for any SCM $\1M \in \2M$, there exists a CCM $\widehat{\1M} \in \widehat{\2M}$ such that $\1N$ approximates the observational distribution and treatment effects in $\widehat{\1M}$ with arbitrary accuracy $\epsilon$. It suffices to construct a CGM $\1N \in \2N$ approximating the CCM $\widehat{\1M}$ with arbitrary accuracy $\epsilon$. 

By the definiton of CCM in \Cref{def:ccm}, given any unit $\*u$, the values of state and action are given by a constant pair $(\*x_i, \*z_i)$, $\widehat{f}_{\*X}(\*u) \gets \*x_i$ and $\widehat{f}_{\*Z}(\*u) \gets \*z_i$. This could be trivially simulated using a conditional distribution $P(\*X, \*Z \mid \*u)$. 

Given a unit $\*u$, values of reward $Y$ in CCM $\widehat{\1M}$ are given by a simple function mapping from the state-action domain $\3R^{n+l}$ to reward domain $\3R$. That is, $\widehat{f}_Y(\*x, \*z, \*u) \gets \sum_{i}^N y_i \I_{\3X_i}(\*x)\I_{\3Z_i}(\*z)$. One could immediately construct a continuous function $h_{i}: \3R^{n+l} \mapsto \3R$ that approximates the above step function with arbitrary accuracy. The construction follows the same step as \citep[Prop.~3.48]{axler2020measure}. This completes the proof.
\end{proof}

\section{Experimental Setups} \label{appendix:b}
In this section, we will provide details on the simulation setups. We also conduct additional experiments on other SCM instances. For all experiments, we collect $1000$ observational samples and use them to draw $1,000$ posterior samples of the treatment effects; each sampled treatment effect function contains potential outcomes $Y_{\*x}$ valued at $1000$ different action assignments. All experiments are performed on a laptop with an Intel Core i9-12900H processor, an RTX 3080 GPU, and 64GB of memory. We use PyMC \citep{abril2023pymc} as the computational framework for probabilistic programming.

For all experiments, exogenous probabilities $P(u)$ are drawn from a uniform Dirichlet prior $\texttt{Dirichlet}\left (1, \dots, 1\right)$. For the treatment assignment distribution $P(\*X, \*Z|u)$, the mean $\*\mu_u$ and variance $\*\Sigma_{u}$ are drawn from a standard normal distribution $\texttt{Normal}(0, 1)$ and a half normal distribution $\texttt{Half-Normal}(0, 0.05)$ respectively. As for the prior distribution over the reward function $h_u(\*x)$, we use the radial basis function kernel with a length scale $\lambda_u$ drawn from a half normal distribution $\texttt{Half-Normal}(0, 5)$.

\paragraph{International Stroke Trial.} We filter the original IST data $\1D$ to create two different datasets, one used as the experimental data $\1D_{\text{exp}}$ and the other as the confounded observational data $\1D_{\text{obs}}$. This evaluation procedure allows us to simulate significant confounding bias using a real-world dataset. Specifically, for every patient $(x_i, y_i, z_i, u_i) \in \1D$ in the IST data, we include it in the subset $\1D_{\text{exp}}$ if one of the following conditions holds:
\begin{align}
    \begin{cases}
        140 < x_i < 180 \text{ and } y_i < 0.2\\
        (x \leq 140 \text{ or } x \geq 180) \text{ and } 0.2 < y_i < 0.5
    \end{cases}
\end{align}
Doing so allows us to generate a filtered dataset $\1D_{\text{exp}}$. For the sake of experiments, we will use $\1D_{\text{exp}}$ as the experimental data where there is no unmeasured confounding affecting the age $X$ and the subsequent probability of death $Y$. To simulate the presence of unobserved confounding, we will filter the patient samples $(x_i, y_i, z_i, u_i) \in \1D_{\text{exp}}$. Specifically, we include a patient $(x_i, y_i, z_i, u_i)$ in a new dataset $\1D'$ if one of the following conditions holds:
\begin{align}
    \begin{cases}
        x < 140 \text{ and } 16 \leq 30 u_i + z_i \leq 36\\
        140 < x_i < 180 \text{ and }121 \leq 30 u_i + z_i \leq 128 \\
        x \geq 180 \text{ and } 27 \leq 30 u_i + z_i \leq 49 
    \end{cases}
\end{align}
Doing so allows us to create a new filtered dataset $\1D'$. Finally, we drop the column containing the patient's conscious state $U$ and use it as the observational data $\1D_{\text{obs}}$. Due to the selection rules described above, distributions generating datasets $\1D_{\text{exp}}$ and $\1D_{\text{obs}}$ differ significantly. These selection rules simulate the causal mechanisms of unobserved confounding in practical domains \citep{kallus2018confounding,zhang2021bounding}.

\paragraph{Estimating Negative Log Predictive Likelihood.}Obtaining the exact estimation of the NLPL score $\1S$ in \Cref{eq:_nlpl} is challenging since it involves integration over the continuous space of model parameters $\*\theta$. To address this issue, we will compute the empirical estimates $\widehat{\1S}$ of the NLPL score. Specifically, given a sequence of state-action pairs in the experimental data $\bar{\*x^*} = (\*x^*_i)_{i = 1, \dots, N^*}$ and $\bar{\*z^*} = (\*z^*_i)_{i = 1, \dots, N^*}$, we samples $1000$ sequences of predictive outcomes $\hat{\*y}^{(j)} = (y^{(j)}_i)_{i = 1, \dots, N^*}$, and outcome variance $\sigma^2_j$, $j = 1, \dots, 1000$, from a posterior distribution $P(\*\theta \mid \1D)$. Given a predictive sequence $\hat{\*y}^{(j)}$ and a variance $\sigma^2_j$, the predictive likelihood over the actual experimental outcome $\*y^* = (y^*_i)_{i = 1, \dots, N^*}$ is defined as the likelihood of a multi-variate Gaussian model as follows:
\begin{align}
    P\left(\*y^* \mid \doo(\*x^*), \*z^*, \hat{\*y}^{(j)}, \sigma^2_j \right) = \prod_{i = 1}^{N^*} \frac{1}{\sqrt{2 \pi \sigma^2_j}} \exp\left(-\frac{\left(y^*_i - \hat{y}_i^{(j)}\right)}{2\sigma^2_j}\right)
\end{align}
The empirical estimate $\widehat{\1S}$ of the NLPL score is given by the neg log of the empirical mean of the above Gaussian likelihood function. That is, given predictive sequences $\hat{\*y}^{(j)} = (y^{(j)}_i)_{i = 1, \dots, N^*}$ and  variances $\sigma^2_j$, $j = 1, \dots, N'$,
\begin{align}
    \widehat{\1S} = - \log \left (\frac{1}{N'}\sum_{j = 1}^{N'} P\left (\*y^* \mid \doo(\*x^*), \*z^*, \hat{\*y}^{(j)}, \sigma^2_j \right)\right)
\end{align}
The above estimate $\widehat{\1S}$ is consistent, since the sampled predictive sequence $(\hat{\*y}^{(j)}, \sigma^2_j)_{j = 1, \dots, N'}$ already includes the posterior variance of the Gaussian processes.

\end{document}